\documentclass{article}
\usepackage[preprint]{tmlr}
\usepackage{amsmath,amssymb,amsfonts}
\usepackage{amsthm}
\newtheorem{proposition}{Proposition}
\newtheorem{lemma}[proposition]{Lemma}
\newtheorem{corollary}[proposition]{Corollary}
\usepackage{booktabs}
\usepackage{longtable}
\usepackage{graphicx}
\usepackage{xcolor}
\usepackage{hyperref}
\usepackage{url}
\usepackage{tikz}
\usetikzlibrary{arrows.meta,positioning,calc}

\usepackage{amsmath,amsfonts,bm}

\def\eqref#1{equation~\ref{#1}}

\def\1{\bm{1}}

\DeclareMathAlphabet{\mathsfit}{\encodingdefault}{\sfdefault}{m}{sl}
\SetMathAlphabet{\mathsfit}{bold}{\encodingdefault}{\sfdefault}{bx}{n}

\newcommand{\R}{\mathbb{R}}

\usepackage{enumitem}
\setlist{nosep}

\graphicspath{{figures/}}

\makeatletter
\if@preprint
  \newcommand{\codefootnote}{\footnote{\codeurl}}
  \newcommand{\pipelineloc}{released at \codeurl}
\else
  \newcommand{\codefootnote}{\footnote{Code and results are included in the supplementary material.}}
  \newcommand{\pipelineloc}{included in the supplementary material}
\fi
\makeatother

\title{The Communication Map of a Transformer}

\author{\name Richard Zhe Wang \email zwang@sjf.edu \\
      \addr St.\ John Fisher University}

\def\month{MM}  
\def\year{YYYY}
\def\openreview{\url{https://openreview.net/forum?id=XXXX}}

\begin{document}

\maketitle

\begin{abstract}
The components of a transformer communicate by writing to and reading
from a shared residual stream, and the mechanistic interpretability
literature has mapped these connections by hand, one circuit at a
time. We present the \emph{communication map}, which charts every potential
communication channel from the geometry of the model's weights alone. Its unit of analysis is the whole transformer rather than the
individual head or circuit. Every ordered writer--reader pair, across all
$18$ connection classes from head-to-head to neuron-to-neuron, receives a
\emph{coupling coefficient} that generalizes the composition score of
Elhage et al.\ (2021), with an exact chance level. The census finds that
$70$--$89\%$ of head pairs are oriented far from chance, some coupled
strongly and others actively avoiding each other. The map recovers the known induction circuits blind from the strongest head-to-head couplings and groups the heads into communities, and ablating one such community destroys the model's in-context copying. As an application, we pool the coupling coefficients of every head to identify a two-dimensional
residual stream subspace, distinct from those found by activation PCA or
outlier dimensions, whose deletion abolishes induction in six models up to
Pythia-6.9B. We release the map, the statistical machinery, and the
intervention suite.\codefootnote
\end{abstract}

\section{Introduction}
\label{sec:intro}

A transformer \citep{vaswani2017attention} connects all of its components through a single shared medium of communication, the \emph{residual stream} \citep{elhage2021mathematical}. Every attention head and every neuron reads from the residual stream and writes its result back for the components downstream to read. Superposition allows this $d$-dimensional stream to hold exponentially many nearly orthogonal directions, far more than $d$. The residual stream therefore acts as a public information ``highway'' carrying many millions of ``channels'' of communication between components \citep{elhage2022superposition}.

However, little is known about how the hundreds of attention heads, the orders of magnitude more neurons, and the other components of a transformer communicate over this information highway. \textit{Who is talking to whom? Do heads in one layer mainly talk to the heads in the next layer up or do they communicate long-range across multiple layers? Do they form subnetworks? Do they share common channels in a ``broadcast'' style, or do they prefer one-to-one communication over more ``private'' channels?} These questions go to the heart of how transformers work. To our knowledge, they largely remain unanswered in the mechanistic interpretability literature \citep{ferrando2024primer,rai2024practical}. This paper answers these questions by building the first communication map of an entire transformer from the geometry of the model's weight matrices alone.

Mechanistic interpretability has examined individual circuits and connections of a transformer. \citet{elhage2021mathematical} defined
\emph{composition scores} measuring whether one attention head reads what
another wrote, and used them to explain \emph{induction heads} \citep{olsson2022induction}, the heads that let a model copy the token that followed an earlier occurrence of the current token. \citet{wang2022interpretability}
hand-traced the indirect-object-identification (IOI) circuit, some two dozen heads. \citet{merullo2024talking}
showed that specific pairs of distant heads communicate over low-rank
channels. In every case, the analysis starts from a known behavior and works backward to a handful of components. Our paper proceeds in the opposite direction. We build the map with no behavior chosen in advance and then look for the known circuits inside it. To our knowledge, no communication map of a full transformer has been built before. The closest attempt, by \citet{projectionkernel2026}, computes a subspace-affinity metric between attention heads, but it reports only the top twenty pairs and excludes MLPs.

The residual stream is akin to a wide radio band that accommodates millions of nearly interference-free channels of communication simultaneously. Each component transmits and receives in its own particular subspaces, like its own ``frequencies,'' so when the directions that a head in Layer 1 writes to align closely with the directions that a head in Layer 5 reads from, the two heads can pass information directly. A trained model's weight matrices therefore contain, implicitly, a complete directory of who can talk to whom, over which subspaces, and how strongly. \textit{Using that directory to reconstruct the communication map of an entire transformer, and to demonstrate the map's usefulness in a new application, is the central goal of this paper.}

To demonstrate this mapping technique, we build full-transformer communication maps for GPT-2 small, medium, and large \citep{radford2019language}, for GPT-Neo-125M \citep{black2021gptneo}, and for the Pythia family (160m, 2.8B, 6.9B) \citep{biderman2023pythia}. We map every connection among all causally eligible attention heads, MLP neurons, and embedding, positional, and unembedding matrices in each transformer. Every directed writer--reader pair is scored from the weights with the \emph{coupling coefficient} $C = \lVert RW \rVert_F / (\lVert R \rVert_F \lVert W \rVert_F)$, generalized from \citet{elhage2021mathematical} to writer--reader pairs of any shape, where $W$ is the writer's write matrix and $R$ the reader's read matrix. The result is a graph of \emph{potential connectivity} for the whole model. We validate the map by ablation against a circuit whose function is known, and then apply it to a new problem.

\citet{elhage2021mathematical} provided the tool for our analysis, the composition score, but the two studies differ in their unit of analysis. They introduced composition scores to test whether one named head reads what another named head wrote, and applied them by hand to a few heads of small attention-only models. We use the same score, which we call the coupling coefficient, to construct the full communication map of a transformer. Their unit of analysis is microscopic, at the level of individual heads and circuits, whereas ours is macroscopic, at the level of all valid writer--reader pairs inside a transformer. At that level the map can ask questions that cannot be posed one circuit at a time: how far communication reaches across layers, whether heads form communities, whether pairs avoid each other as well as couple, and when that structure appears during training.

We contribute to the literature in the following five areas:

\textbf{Full map.} First, we construct the first full communication map of a transformer (Section~\ref{sec:map}), covering all $18$ connection classes and billions of candidate connections. Attention heads communicate over long ranges spanning many layers and form communities. One community holds all five induction heads and two of the three IOI name-mover heads, and ablating it destroys $94.7\%$ of the model's in-context copying. Along the training trajectory of Pythia, super-coupled pairs appear within the first thousand steps and avoidant pairs develop later.

\textbf{Efficient computation at scale.} Second, we develop efficient algorithms that make the map computable. A whole-model map scores billions of candidate pairs, and the naive computation is infeasible at that scale. In arithmetic, multiplying two $d \times d$ matrices per pair costs about $10^{17}$ operations for Pythia-6.9B. In memory, the composed operators would occupy $137$~GB, four times the capacity of a $32$~GB consumer GPU. To resolve this issue, we derive exact factored forms for efficient computations of the map for every connection class, under which every coupling coefficient and every closed-form null moment is assembled from $d_{\mathrm{head}}$-wide objects, a reduction in per-pair cost by $(d/d_{\mathrm{head}})^{2}$, which is $1{,}024$ on Pythia-6.9B. Consequently, the full map builds in $15$ seconds for GPT-2 small and in $11$ minutes for Pythia-6.9B, all on one consumer GPU (Appendix~\ref{app:pipeline}).

\textbf{Statistical footing.} Third, we give the census its statistical footing. The coupling coefficient has an exact chance level and an exact null variance for every connection class, so every score carries a standardized $z$-score, and a randomly initialized model sits at that chance level. We also provide the statistical machinery that separates real connections from chance alignment.
  
\textbf{Induction-critical subspace.} Fourth, in the application of the map (Section~\ref{sec:interventions}), we pool the coupling coefficients of all heads of a transformer to identify a two-dimensional residual stream subspace that is critical to the induction capability of the whole model. Deleting this subspace destroys $82$--$96\%$ of the induction capability on six models from $124$M to $6.9$B parameters. This subspace is distinct from the subspaces that activation PCA and outlier dimensions identify \citep{mu2018allbutthetop,timkey2021all}.

\textbf{Release.} Fifth, we release the map, the statistical machinery, and the intervention suite, with fixed seeds.

\section{Background and related work}
\label{sec:background}

The residual stream is the communication backbone of a transformer. Attention heads and MLP neurons in different layers address one another by writing to and reading from the stream. Each does so through fixed linear maps into subspaces of its own, and all writes to the residual stream combine additively \citep{elhage2021mathematical}. At every
token position $i$, layer $\ell$ receives a stream vector
$x_i \in \mathbb{R}^{d}$ ($d{=}768$ for GPT-2 small) that is exactly
\begin{equation}
x_i \;=\; \mathrm{emb}(t_i) \;+\; \mathrm{pos}(i) \;+\;
\sum_{c\,\in\,\text{components before }\ell} \Delta_c(i),
\label{eq:additivity}
\end{equation}
where $\Delta_c(i)$ is the write of component $c$. Every reader's input
therefore decomposes exactly into contributions from identifiable
writers, all of which share the same $d$ dimensions
\citep{elhage2022superposition}. Each head or MLP neuron is tuned to its own subspaces of the stream. An upstream writer can therefore influence a downstream reader precisely when the writer's subspaces align with those of the reader.
The stream also has three \emph{interface matrices}, where the model meets its inputs and outputs. The token embedding $W_E$ and the positional embedding $W_{\mathrm{pos}}$ write into the stream before layer $0$. The unembedding $W_U$ reads the stream after the last layer and converts the result to next-token logits.


\citet{elhage2021mathematical} establish that an upstream writer head can influence a downstream reader head through three channels, known as the \emph{K-composition}, \emph{Q-composition}, and \emph{V-composition}.
Here $W_{QK} = W_Q W_K^{\top}$ is a head's attention-pattern form, which scores query--key pairs, and $W_{OV} = W_O W_V^{\top}$ is its write operator, which transforms what the head reads into what it writes. The subscript $w$ marks the writer head and $r$ the reader. The \emph{K-composition} is $W_{QK,r} W_{OV,w}$ \citep{olsson2022induction}, the \emph{Q-composition} is $W_{QK,r}^{\top} W_{OV,w}$, and the \emph{V-composition} is $W_{OV,r} W_{OV,w}$.\footnote{A fourth term, in which the writer enters the keys and the queries at once, is bounded by the two one-sided channels and thus adds no independent channel of its own (see Appendix~\ref{app:attnblock} for proof).} The conventions, shapes, and the attention block derivation are in Appendices~\ref{app:conventions}--\ref{app:attnblock}. 

Each MLP neuron is also an individual rank-one reader and writer, reading in one direction through $a_m = r_m^{\top} x$ and writing in another direction via $\Delta_m = g(a_m)\, w_m$ \citep{geva2021transformer}. The neuron literature probes what individual neurons encode and finds families of functionally universal neurons recurring across models \citep{gurnee2023finding,gurnee2024universal}. That literature asks what a neuron represents. The map asks whom a neuron talks to.

A writer--reader pair whose subspaces align far beyond chance is a connection, or ``edge,'' in the communication graph. We quantify connection strength with the coupling coefficient $C$ in Section~\ref{sec:map}.

The literature can be organized along two axes. The first axis is the unit of analysis. A study is microscopic when its unit is a single head, neuron, or circuit, and macroscopic when its unit is the whole model. The second axis is the source of evidence. A study is weight-based when its objects are computed from the parameters alone, and activation-based when its objects are the states the model produces on a corpus of inputs. The two axes give four positions, three of which are well populated. Microscopic weight-based studies include the composition scores of \citet{elhage2021mathematical}, the low-rank head-pair channels of \citet{merullo2024talking}, and the readings of individual weight matrices through their singular vectors or in embedding space \citep{millidge2022svd,dar2023analyzing}. Microscopic activation-based studies include the traced circuits for induction and indirect-object identification \citep{olah2020zoom,olsson2022induction,wang2022interpretability} and their automated successors \citep{conmy2023acdc,syed2024attribution}. Macroscopic activation-based studies include the dictionary features of sparse autoencoders and transcoders, and the attribution graphs built on them \citep{bricken2023monosemanticity,cunningham2024sparse,dunefsky2024transcoders,marks2024sparse,ameisen2025circuit}. These studies aim at the whole model, but they describe it through its activations on a corpus and through a learned dictionary.

\begin{figure}[t]
\centering
\resizebox{0.8\linewidth}{!}{%
\begin{tikzpicture}[
  cell/.style={draw=gray!60, rounded corners=3pt, minimum width=6.4cm,
               minimum height=2.5cm, text width=6.0cm, align=center,
               fill=gray!5, font=\small, inner sep=5pt},
  ours/.style={cell, fill=blue!9, draw=blue!55!black, line width=0.9pt},
  cite/.style={font=\footnotesize, color=black!70},
  axis/.style={font=\small\bfseries},
  tick/.style={font=\small\itshape, color=black!80},
  >=Stealth]
\node[cell] (mw) at (0,0) {%
  composition scores\\[-1pt]
  {\footnotesize\color{black!70}\citep{elhage2021mathematical}}\\[3pt]
  head-pair channels\\[-1pt]
  {\footnotesize\color{black!70}\citep{merullo2024talking}}\\[3pt]
  readings of single weight matrices\\[-1pt]
  {\footnotesize\color{black!70}\citep{millidge2022svd,dar2023analyzing}}};
\node[cell] (ma) at (6.9,0) {%
  hand-traced circuits\\[-1pt]
  {\footnotesize\color{black!70}\citep{olsson2022induction,wang2022interpretability}}\\[3pt]
  automated circuit discovery\\[-1pt]
  {\footnotesize\color{black!70}\citep{conmy2023acdc,syed2024attribution}}};
\node[ours] (Mw) at (0,-3.5) {%
  \textbf{the communication map (this paper)}\\[2pt]
  every writer--reader pair of the whole model,\\ scored from the weights alone\\[3pt]
  {\footnotesize\color{black!70}prior work: head pairs of one model}\\[-1pt]
  {\footnotesize\color{black!70}\citep{projectionkernel2026}}};
\node[cell] (Ma) at (6.9,-3.5) {%
  sparse autoencoders and transcoders\\[-1pt]
  {\footnotesize\color{black!70}\citep{bricken2023monosemanticity,cunningham2024sparse}}\\[-1pt]
  {\footnotesize\color{black!70}\citep{dunefsky2024transcoders}}\\[3pt]
  attribution graphs\\[-1pt]
  {\footnotesize\color{black!70}\citep{marks2024sparse,ameisen2025circuit}}};
\draw[->, thick, black!70] ($(mw.north west)+(0,0.55)$) -- ($(ma.north east)+(0,0.55)$);
\node[tick, anchor=south] at ($(mw.north)+(0,0.62)$) {weights};
\node[tick, anchor=south] at ($(ma.north)+(0,0.62)$) {activations};
\node[axis, anchor=south] at ($(mw.north)!0.5!(ma.north)+(0,1.05)$) {Source of evidence};
\draw[->, thick, black!70] ($(mw.north west)+(-0.55,0)$) -- ($(Mw.south west)+(-0.55,0)$);
\node[tick, rotate=90, anchor=south] at ($(mw.west)+(-0.62,0)$) {microscopic};
\node[tick, rotate=90, anchor=south] at ($(Mw.west)+(-0.62,0)$) {macroscopic};
\node[axis, rotate=90, anchor=south] at ($(mw.west)!0.5!(Mw.west)+(-1.05,0)$) {Unit of analysis};
\end{tikzpicture}}
\caption{\textbf{Two axes organize the literature on transformer internals.} Rows give the unit of analysis and columns the source of evidence. The communication map occupies the fourth position, whose only prior occupant ranks the head pairs of one model \citep{projectionkernel2026}.}
\label{fig:axes}
\end{figure}

The fourth position, macroscopic and weight-based, is the position of the map (Figure~\ref{fig:axes}). The map is macroscopic because its unit of analysis is every causally eligible writer--reader pair in the model, and it is weight-based because every coupling coefficient is computed from the parameters alone. To our knowledge the only prior occupant of this position is the head-pair ranking of \citet{projectionkernel2026}, discussed below, which covers one model and omits neurons. The two coordinates of this position determine how the map should be read. Because the map is weight-based, it charts the channels the weights permit rather than the traffic that a corpus happens to exercise. Because the map is macroscopic, no existing method produces the same object, and we therefore validate the map against circuits whose function is already known.


The activation-based studies share the map's interest in how components interact, but they draw their evidence from the activations a model produces on particular inputs, whereas the map starts from the geometry of the model's weight matrices. These works include automated circuit discovery and attribution \citep{conmy2023acdc,syed2024attribution,ameisen2025circuit,marks2024sparse}, causal interventions \citep{vig2020investigating,meng2022locating,chan2022causal,zhang2024towards}, head ablations \citep{michel2019sixteen,voita2019analyzing}, and scaling studies \citep{lieberum2023does}. They characterize the traffic that flows on particular inputs, one behavior at a time, whereas the map characterizes the channels the weights permit, independently of any input.

The most recent of these methods replace neurons with learned features as the unit of analysis. Sparse autoencoders and transcoders decompose activations into dictionary features \citep{bricken2023monosemanticity,cunningham2024sparse,dunefsky2024transcoders}, and attribution graphs trace how those features feed one another on a given prompt \citep{marks2024sparse,ameisen2025circuit}. The map differs from this line of work in three ways. Its nodes are the model's own heads, neurons, and interface matrices rather than learned dictionary features, so we train no dictionary and incur no reconstruction error. Its input is the weights, so it charts every channel the weights permit, whereas an attribution graph charts the channels one prompt uses. Its scope is the whole model, whereas an attribution graph is built around one behavior. The two are therefore complementary. A feature dictionary's encoder reads the residual stream through a linear map and its decoder writes back through another, so features can enter the map as readers and writers exactly as heads and neurons do. We leave that feature-level map to future work (Section~\ref{sec:conclusion}).

\citet{tang2026concept} patch label-aligned subspaces of the residual stream and show that a roughly 70-dimensional subspace carries most of the causal effect of in-context task information in Llama-3-8B. \citet{xu2026closure} cluster attention heads into communities by co-activation and validate them by ablation against matched random controls. Section~\ref{app:communityrep} shares that validation logic, but our communities are built from the geometry of the weights rather than from co-activation, and the map also covers neurons and the interface matrices.

The closest weight-geometry work to ours is by \citet{projectionkernel2026}, who rank head-pair subspace affinities in GPT-2 small and visualize the top twenty. Three differences separate their analysis from ours. First, they exclude MLPs, whereas the map covers every neuron. Second, they discard the operators' singular values, the channel gains, whereas the coupling coefficient weights every direction by its gain (Eq.~\ref{eq:cosform}). Third, they compare pooled score distributions to a random-subspace reference, whereas the map tests each pair individually against its own null distribution.

Assembling these primitives into a whole-model map is also a statistics problem. Screening all of a model's candidate channels, for example $6.3 \times 10^{8}$ in GPT-2 small and $1.3 \times 10^{11}$ in Pythia-$6.9$B, is large-scale simultaneous inference. We apply the toolbox that statisticians developed for genomic screening, empirical null distributions fitted to the data's own geometry \citep{efron2004large} and the false-discovery-rate (FDR) control procedure of \citet{benjamini1995controlling}.

\section{The communication map}
\label{sec:map}

To create the full-transformer communication map, we compute a \textit{coupling coefficient} for every ordered, causally eligible writer--reader pair within a model. The census of these coefficients forms the basis of the map.
Derivations are in Appendix~\ref{app:math}, and the computational methods are in Appendix~\ref{app:pipeline}.

\subsection{The coupling coefficient}
\label{sec:coeff}

Generalizing the composition score of \citet{elhage2021mathematical}
from head pairs to arbitrary writer--reader matrices in a transformer, we define the
\emph{coupling coefficient} of a writer with write matrix $W$ and a reader with read matrix $R$ as
\begin{equation}
C \;=\; \frac{\lVert R\,W \rVert_F}{\lVert R \rVert_F\, \lVert W \rVert_F}
\qquad\in[0,1].
\label{eq:statistic}
\end{equation}
The reader and writer can be any component that writes to or reads from the residual stream. For example, $W$ equals $W_{OV,w}$ for a head and $w_m$ for a neuron. Likewise, $R$ equals $W_{QK,r}$ for K-composition, $W_{QK,r}^{\top}$ for Q-composition, $W_{OV,r}$ for V-composition, and $r_m^{\top}$ for a neuron. The token embedding matrix $W_E$ and, when it exists, the positional embedding matrix $W_{pos}$ are also writers, whereas the unembedding matrix $W_U$ is the final reader.\footnote{As in Section~\ref{sec:background}, we call these three matrices ($W_E$, $W_U$, and $W_{pos}$ if it exists) the \textit{interface matrices}, as they handle the model's token inputs and final token outputs.} For head--head pairs, $C$ is exactly the composition score of \citet{elhage2021mathematical}. One formula thus covers every connection class of the map, from the $d \times d$ head circuits down to rank-one neurons.

\paragraph{Geometric interpretation.}
Let $p_k$ be the fraction of the reader's total squared weight $\lVert R \rVert_F^{2}$ carried by its $k$-th principal input direction, that is, its $k$-th squared singular value over their sum. Let $q_\ell$ be the same fraction for the writer's output directions, and let $\theta_{k\ell}$ be the angle between the two directions. Then
Proposition~\ref{prop:factor} (Appendix~\ref{app:statistic}) says:
\begin{equation}
C^2 \;=\; \sum_{k,\ell} p_k\, q_\ell\, \cos^2\theta_{k\ell}.
\label{eq:cosform}
\end{equation}
To use a radio analogy in which each direction is a frequency band, $\cos^2\theta_{k\ell}$
asks whether the two stations are tuned to the same frequency,
$q_\ell$ measures how strongly the sender transmits on that frequency, and
$p_k$ measures how much the receiver amplifies the signal. Since $p$ and $q$ each sum to one, $C^2$ is a weighted average of $\cos^2\theta_{k\ell}$, so the coupling is high only when the sender's transmission strength and the receiver's amplification concentrate on the same direction pairs and those directions are tuned alike.

\paragraph{Properties of the coupling coefficient.}
Three properties of $C$ support the census (proofs in Appendix~\ref{app:statistic}). First, $C$ is rotation invariant. It depends on the pair only through the Gram matrices $G = R^{\top} R$ and $H = W W^{\top}$ (Proposition~\ref{prop:gram}, Corollary~\ref{cor:invariance}), so a rotation of the $d_{\mathrm{head}}$-dimensional head space, which changes a head's internal bookkeeping but not its effect on the residual stream, leaves $C$ unchanged. Second, $C \in [0, 1]$, and for rank-one pairs it collapses to $|\cos|$, so neuron--neuron pairs and head pairs are measured on a common scale (Corollary~\ref{cor:bounds}). Third, under a uniformly random (Haar) rotation of either side's subspace, $\mathbb{E}[C^2] = 1/d$ exactly, for any singular values, giving one baseline for all classes of connections at once (Proposition~\ref{prop:chance}), and for rank-one pairs $C^2$ follows a $\mathrm{Beta}\big(\tfrac12, \tfrac{d-1}{2}\big)$ distribution (Lemma~\ref{lem:beta}).

\subsection{The census}
\label{sec:census}
\paragraph{Computing the map.} We compute the map for seven LLMs of varying scales in three model families, GPT-2, GPT-Neo, and Pythia. The candidate pairs number $6.3\times10^{8}$ for GPT-2 small and $1.3\times10^{11}$ for Pythia-6.9B, with the counts per class in Table~\ref{tab:scalecensus} of Appendix~\ref{app:pipeline}.
Each pair's $C^{2}$ is standardized against its \emph{rotation null distribution}, the distribution of $C^{2}$ under the random rotation of property 3, whose mean and variance are closed-form (Appendix~\ref{app:nulls}). The standardized score of a pair is
\begin{equation}
z \;=\; \frac{C^{2} - 1/d}{\sqrt{\operatorname{Var}[C^{2}]}},
\label{eq:zmain}
\end{equation}
where the mean $1/d$ is Proposition~\ref{prop:chance} and the variance is the exact expression of Eq.~\ref{eq:theovar}, so no sampling or approximation enters. Tables~\ref{tab:theorycensus} and \ref{tab:censustraining} and Appendix~\ref{app:nulls} also call this distribution the theoretical null distribution. Computing $C$ naively would multiply two $d \times d$ matrices for each of the billions of candidate pairs. Instead, we use the cyclic property of the trace, under which every class collapses to much smaller $d_{\mathrm{head}}$-wide products (Appendix~\ref{app:pipeline}). The full map therefore builds in $15$ seconds for GPT-2 small and in $11$ minutes for Pythia-6.9B on one consumer GPU.

\begin{table}[h]
\centering
\caption{\textbf{Head-to-head pairs against the theoretical null
distribution}, for every mapped model. Share of head pairs
significantly above ($z \ge 2$) and below ($z \le -2$) chance
orientation, per composition channel, with the median $z$ of each group in
parentheses. Here $z$ is the standardized score of Eq.~\ref{eq:zmain}, with both moments closed-form (Appendix~\ref{app:nulls}).}
\label{tab:theorycensus}
\footnotesize
\begin{tabular}{lccc}
\toprule
& K-composition & Q-composition & V-composition \\
Model & above / below & above / below & above / below \\
\midrule
GPT-2 & $55\%\,(+16)$ / $34\%\,(-8)$ &
  $61\%\,(+18)$ / $27\%\,(-7)$ &
  $25\%\,(+15)$ / $65\%\,(-10)$ \\
GPT-2-medium & $54\%\,(+12)$ / $28\%\,(-7)$ &
  $62\%\,(+15)$ / $20\%\,(-6)$ &
  $25\%\,(+13)$ / $58\%\,(-9)$ \\
GPT-2-large & $63\%\,(+13)$ / $23\%\,(-6)$ &
  $70\%\,(+14)$ / $15\%\,(-5)$ &
  $32\%\,(+12)$ / $51\%\,(-7)$ \\
GPT-Neo-125M & $40\%\,(+8)$ / $39\%\,(-6)$ &
  $48\%\,(+9)$ / $33\%\,(-7)$ &
  $24\%\,(+8)$ / $57\%\,(-8)$ \\
Pythia-160m & $50\%\,(+7)$ / $27\%\,(-7)$ &
  $65\%\,(+9)$ / $16\%\,(-5)$ &
  $21\%\,(+9)$ / $63\%\,(-10)$ \\
Pythia-2.8B & $55\%\,(+9)$ / $14\%\,(-4)$ &
  $52\%\,(+8)$ / $12\%\,(-3)$ &
  $27\%\,(+9)$ / $50\%\,(-7)$ \\
Pythia-6.9B & $84\%\,(+14)$ / $4\%\,(-4)$ &
  $50\%\,(+7)$ / $10\%\,(-3)$ &
  $31\%\,(+11)$ / $50\%\,(-7)$ \\
\bottomrule
\end{tabular}
\end{table}

\begin{figure}[t]
\centering
\includegraphics[width=0.54\linewidth]{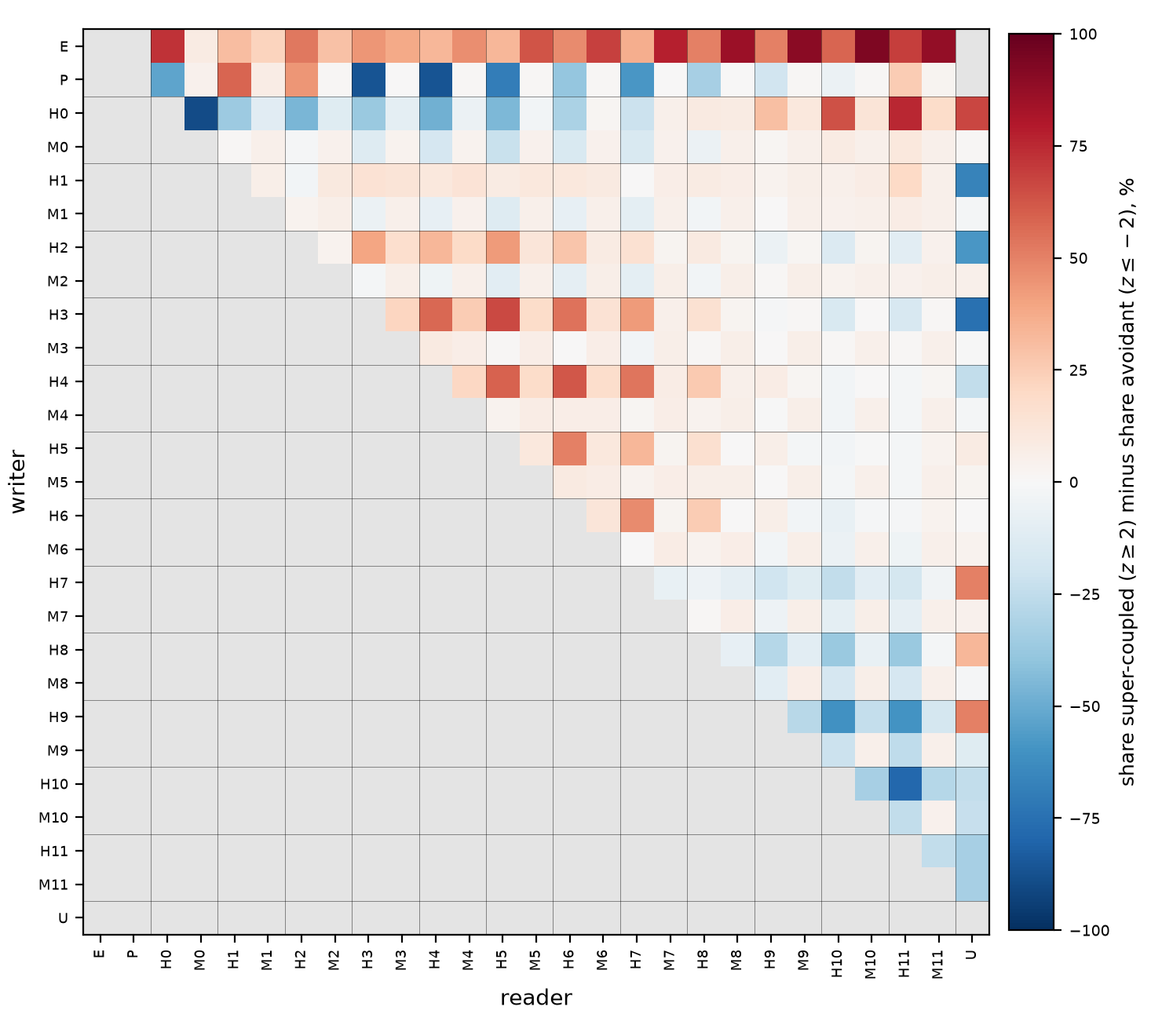}
\caption{\textbf{The communication map of GPT-2 at layer resolution.} Components are grouped into the embedding (E), the positional matrix (P), the heads (H) and neurons (M) of each layer, and the unembedding (U). Writers run down the side and readers across. Each block covers every causally eligible pair of one connection class between two groups, with the three head channels pooled, and shows the share of pairs super-coupled ($z \ge 2$) minus the share avoidant ($z \le -2$) against the rotation null distribution of Eq.~\ref{eq:zmain}. Red is net super-coupling, blue net avoidance, white chance, and grey causally ineligible. All $18$ connection classes appear in the one matrix.}
\label{fig:layeroverview}
\end{figure}

The rotation null distribution is the census baseline, but it is not the standard by which edges are selected. It describes the coupling that a random relative orientation of the read and write subspaces would produce with the singular values held fixed, and trained weights are expected to sit far from it. The empirical null distribution used to select edges (Section~\ref{app:selection}) is instead fitted to the observed scores themselves, and asks which pairs are exceptional even among trained pairs of their own kind. Table~\ref{tab:theorycensus} reveals several patterns. Most K- and Q-compositions have $C^2$ values that are either significantly above ($z \ge 2$) or below ($z \le -2$) the chance level. We call a head pair with $z \ge 2$ \textit{super-coupled} and a head pair with $z \le -2$ \textit{avoidant}, since the two heads of an avoidant pair communicate in subspaces closer to orthogonal than chance would produce.
For instance, $89\%$ of GPT-2 small's head pairs in the K-composition deviate beyond two standard deviations from the chance level, with $55\%$ super-coupled and $34\%$ avoidant. Q-composition follows a similar divide. The same pattern replicates across all seven mapped models. Averaged over the three channels, $70$--$89\%$ of each model's head pairs sit beyond two standard deviations from chance.

The V-composition, in contrast, is more avoidant ($65\%$ in GPT-2 small) than super-coupled ($25\%$). \citet{elhage2021mathematical} observed that V-composition is rare in their toy-scale attention-only models, and our evidence is consistent with that observation. We find active suppression in the V-composition channel at every scale we map, from GPT-2 small to Pythia-$6.9$B, and why it holds more avoidant pairs than super-coupled ones is an open question.
%

Figure~\ref{fig:layeroverview} draws the same census for every class at once, at the resolution of layers. Four features are visible there and in no single-circuit analysis. First, the embedding is read above chance by the heads and neurons of every layer, and the neurons read it more strongly with depth (net super-coupling of $8\%$ of pairs in layer $0$ against $88\%$ in layer $11$). Second, the positional matrix is read by the heads of layers $1$ and $2$ and avoided by almost every other head, including the heads of layer $0$. Third, the heads of layer $0$ avoid the heads of layers $1$ to $7$ and the neurons of their own layer, but couple strongly to the heads of layers $10$ and $11$ and to the unembedding, so layer $0$ writes for the end of the network. Fourth, the super-coupled population of Table~\ref{tab:theorycensus} sits in a band where the heads of layers $2$ to $6$ write to the next one to three layers, whereas the avoidant population sits among the heads of layers $7$ to $11$, which avoid one another and split on the unembedding. The neuron blocks are near white, because the bulk of neuron pairs sits at chance and the wires described next occupy a tail that this resolution cannot show.

\paragraph{Neuron wires.}
The mapping census extends to the $6.2\times10^{8}$ neuron$\to$neuron
pairs of GPT-2 and the $1.3\times10^{11}$ of Pythia-$6.9$B. For these rank-one pairs $C$ is the cosine $|\cos|$ between the writer's write vector and the reader's read vector, and we call a pair with $|\cos| \ge 0.5$ a \emph{wire} (Table~\ref{tab:scalecensus}). Under the null distribution of $C=|\cos|$, the expected number of pairs with $|\cos| \ge 0.2$ is $14$ and with $|\cos| \ge 0.23$ essentially zero (Corollary~\ref{cor:ceiling}). In reality, GPT-2 small holds $140{,}927$ pairs whose $|\cos| \ge 0.2$, and $78{,}861$ whose $|\cos| \ge 0.23$. Neurons are therefore directly wired to each other. On natural text the wired pairs co-activate at a median absolute correlation of $0.31$ against $0.03$ for random pairs, so the wires are not dormant (Appendix~\ref{app:nulls}). Similar patterns persist at every model scale tested (Table~\ref{tab:scalecensus}; Figure~\ref{fig:nncos}, Appendix~\ref{app:nulls}).

\paragraph{Initialization and training.} The rotation null distribution is also the generating distribution of a randomly initialized model, so a census at initialization calibrates the baseline. At initialization every channel of every model sits at chance, with only $2$--$3\%$ of pairs beyond $|z| = 2$ on either side. Along the public training trajectories of Pythia-160m and Pythia-2.8B \citep{biderman2023pythia}, super-coupling appears within the first thousand steps, where $61$--$84\%$ of K- and Q-composition pairs are already above chance while none is below. Avoidance appears only between steps $1{,}000$ and $4{,}000$ and keeps growing to the end of training. Avoidance is therefore learned, and learned later than coupling (Appendix~\ref{app:training}, Table~\ref{tab:censustraining}, Figure~\ref{fig:censustraining}).

\subsection{Selecting edges}
\label{app:selection}
The census describes the whole population of couplings. The map's edges are the individual pairs coupled far beyond the typical coupling of their own kind. The theoretical null distribution cannot serve as the selection standard, because $47\%$ of GPT-2's head pairs, averaged over the three channels, are significant against it at $z \ge 2$ (Table~\ref{tab:theorycensus}), so selecting by it would draw nearly half of all candidate pairs as edges. We therefore score each pair against the \emph{empirical null distribution} of its stratum \citep{efron2004large}, a stratum being the candidate edges of one connection class at one writer-to-reader layer separation, using the robust $z$-score
\begin{equation}
z_e \;=\; \frac{C_e - \operatorname{med}_s}{1.4826 \cdot
\operatorname{MAD}_s},
\qquad
\operatorname{MAD}_s = \operatorname{med}_{e' \in s}
\big| C_{e'} - \operatorname{med}_s \big|,
\label{eq:empnull}
\end{equation}
where $\operatorname{med}_s$ is the stratum's median and $1.4826$ times the median absolute deviation (MAD) is its robust standard deviation \citep{hampel1974influence,rousseeuw1993alternatives}. Within each connection class, each candidate edge receives the upper-tail $p$-value $p_e = 1 - \Phi(z_e)$, and the $m$ $p$-values are sorted as $p_{(1)} \le \cdots \le p_{(m)}$. Benjamini--Hochberg control at $q = 0.05$ \citep{benjamini1995controlling} then selects the $k^{\star}$ smallest, where $k^{\star} = \max\{k : p_{(k)} \le kq/m\}$, so the edges the map draws carry a false discovery rate (FDR) of at most $5\%$ per class. Appendix~\ref{app:selectiondetails} gives the strata, the central-matching fit, and the case for robust estimators.

\subsection{The head graph of GPT-2}
\label{sec:scan}\label{sec:recovery}\label{sec:headgraph}

\begin{figure}[t]
\centering
\includegraphics[width=0.68\linewidth]{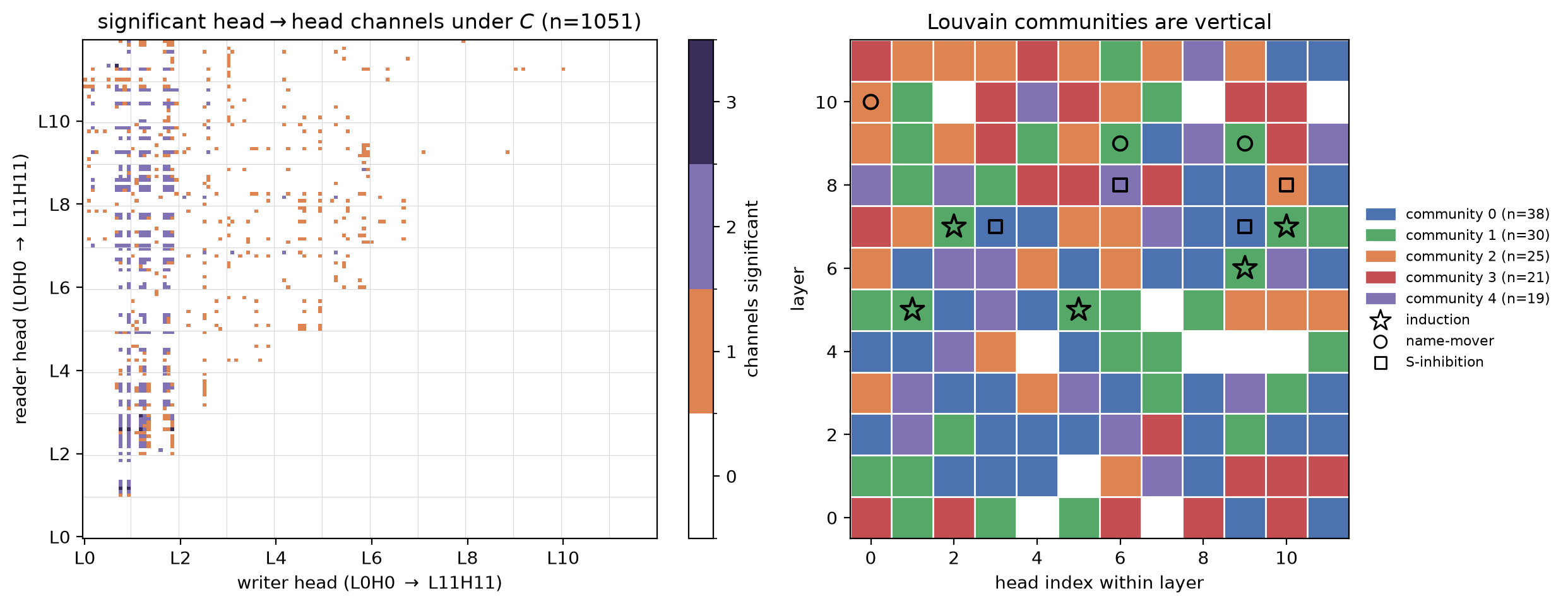}
\caption{\textbf{The head-to-head subgraph of the communication map (GPT-2).}
\emph{Left:} adjacency of the $1{,}051$ selected head$\to$head
channel edges under $C$, one cell per ordered pair, darker where more
of the three channels are selected, layer 0 at the lower left.
\emph{Right:} the seed-$0$ Louvain partition of the same graph,
one cell per head and one color per community, with the five induction
heads ($\star$), the IOI name-mover heads ($\bullet$), and the IOI
S-inhibition heads ($\blacksquare$) marked. White cells are heads outside the giant connected component.}
\label{fig:mapoverview}
\end{figure}

Figure~\ref{fig:mapoverview} (Left) shows the head-to-head graph of GPT-2's $144$ attention heads. Each ordered head pair carries three separately tested channels (K-, Q-, and V-composition), and $1{,}051$ of the $28{,}512$ candidate edges are selected at FDR $q = 0.05$. Layers 0--1 mostly hold \textit{broadcasting} heads, whose writes are read by many later heads, whereas the middle and later layers mostly hold \textit{receiver} heads. A Louvain partition of the graph \citep{blondel2008fast} yields the five communities shown in Figure~\ref{fig:mapoverview} (Right), and the partition is stable across ten Louvain seeds (Appendix~\ref{app:communities}). Three patterns emerge. \textit{First}, the communities are vertical. They span the depth of the model rather than a single layer, and no community spans fewer than eight of GPT-2's twelve layers. \textit{Second}, communication is long-range. The median writer-to-reader separation is five layers against four for the eligible pairs, and the rate of linked pairs roughly doubles with distance (Appendix~\ref{app:span}), consistent with the predictions of \citet{elhage2021mathematical}. \textit{Third}, the known circuits fall out of the graph. The green community of $30$ heads holds all five induction heads and two of the three IOI name-mover heads in ten of ten seeds, where the induction and IOI heads are identified behaviorally following the protocols of prior work \citep{olsson2022induction,wang2022interpretability}. The previous-token head L4H11 is the top-ranked K-composition writer into every induction head. The map thus recovers the induction wiring blind.

\subsection{Validation by ablation}
\label{app:communityrep}

The third pattern above raises the question of whether the community that concentrates the induction heads also carries the model's induction capability. The goal is not to re-identify the induction heads, which are already identified behaviorally, but to test whether the community boundary that the map drew from weights alone encloses the machinery of induction, including members that single-head identification does not flag. We therefore ablate the induction-concentrated community (the green community in Figure~\ref{fig:mapoverview}) and measure the effect on the model's in-context copying capability. We ablate this community, and not the others, because its function is known in advance.

\paragraph{Protocol and results.}
\label{sec:demo2}\label{sec:universality2}%
The readout is the \emph{induction gain}. Each prompt is a BOS token followed by a block of $T = 128$ uniform-random tokens and then an exact copy of that block. Let $\ell_t = -\log p(x_{t+1} \mid x_1, \ldots, x_t)$ be the next-token loss at position $t$. The induction gain is
\begin{equation}
G \;=\; \frac{1}{|B_1|} \sum_{t \in B_1} \ell_t \;-\; \frac{1}{|B_2|} \sum_{t \in B_2} \ell_t,
\label{eq:gain}
\end{equation}
where $B_1$ and $B_2$ are the predicted positions of the first and second block, averaged over $32$ prompts with independently drawn blocks. Random tokens cannot be predicted from the weights, but every token of the second block appeared $T$ positions earlier in the same prompt, so the gain isolates in-context copying. For an intervention, we report the share of the clean gain it destroys, $1 - G_{\mathrm{intervened}}/G_{\mathrm{clean}}$. We ablate a head set by freezing each head's attention-weighted value vector at its corpus mean, $z_h(t) \to \bar{z}_h$ (Appendix~\ref{app:ablationfull}), and compare the set against five random control sets matched in size and per-layer head counts. The natural-text loss on a sample of the Pile is tracked as a check on nonspecific damage. On GPT-2 small, freezing the full thirty-head community's outputs destroys $94.7\%$ of the model's induction gain, against a median of $1.9\%$ for the matched random controls. The ablated community is therefore responsible for virtually all induction in the model. The five behaviorally identified induction heads alone account for only $60.8\%$ of the gain, so the community's remaining members, individually near-inert, carry redundant machinery that single-head identification misses. The ablation replicates across six models, from GPT-2 small to Pythia-2.8B (Table~\ref{tab:communityrep} in Appendix~\ref{app:ablationfull}). At the larger models the community can hold more than half of a layer's heads, so the appendix reports both layer-matched and size-matched controls.

\section{Application: the induction-critical subspace}
\label{sec:interventions}
\label{sec:finding}
\label{sec:unplug}
\label{sec:universality1}

So far the map has been read one edge at a time, and its unit of analysis has been the individual head pair, grouped into communities. The application pools the coupling of every head at once, so the unit of analysis becomes a direction of the residual stream itself. The head graph asks which components talk to which, whereas the application asks over which directions of the stream the whole transformer talks. The goal is to identify a few induction-critical dimensions of the residual stream, shared by the entire transformer, and to conduct ablation experiments on them. This application also connects the map to the literature on important directions in the residual stream, activation PCA and outlier dimensions (Comparison subspaces below).

\subsection{Reader--writer PCA}
\label{app:bandselect}

\paragraph{Overview.}
We develop a method, called \emph{reader--writer PCA} (RW-PCA), to identify an induction-critical subspace in the residual stream of a transformer from the communication map directly. We exclude the MLP neurons, since this analysis focuses on attention heads. RW-PCA shares its operation, an eigendecomposition, with activation PCA \citep{mu2018allbutthetop,ethayarajh2019contextual}, but it decomposes a matrix built from the transformer's reader and writer matrices alone rather than activations sampled on a corpus.

\paragraph{The RW-PCA method.} An attention head touches the stream
through its four factor matrices
$W_Q, W_K, W_V, W_O \in \mathbb{R}^{d \times d_{\mathrm{head}}}$,
each reading or writing a $d_{\mathrm{head}}$-dimensional slice of
it (Appendix~\ref{app:conventions}), and the interface matrices
touch it directly. A unit stream direction $v$ is treated as a writer and scored against any one of these matrices $X$ with Eq.~\ref{eq:statistic}. The squared coupling is then
$C^{2}(X, v) = \lVert X^{\top} v \rVert^{2} / \lVert X
\rVert_F^{2} = v^{\top} G_X\, v / \operatorname{tr}(G_X)$ with
$G_X := X X^{\top}$, a quadratic form in the unit-trace Gram. We then sum this coupling over every attention-head factor and interface matrix,
$\sum_X C^{2}(X, v) = v^{\top} S v$, which defines the \emph{pooled coupling matrix} $S = \sum_{X} G_X/\operatorname{tr}(G_X)$. The trace normalization gives each factor one vote. The stream directions that matter globally are those that maximize $v^{\top} S v$ over orthonormal directions, which is the Rayleigh quotient problem, so they are the eigenvectors of $S$.
These eigenvectors are the global communication \emph{bands} that most heads share, and the ten leading bands form the candidate pool for the selection rule below.

The eigenvalue $\lambda_i = \sum_X C^{2}(X, v_i)$ is the total coupling that band $v_i$ captures across the transformer. Since each factor contributes trace $1$ to $S$, the eigenvalues partition the model's total coupling budget among the bands, and $\lambda_i / \operatorname{tr}(S)$ is the fraction carried by band $v_i$. The eigenvalues decay smoothly with no privileged cutoff (Figure~\ref{fig:bands}, Appendix~\ref{app:universality}).

\paragraph{Selecting the most induction-critical bands.} Induction is \textit{positional}, as induction circuits generally match the current token to its previous occurrence and copy what followed \citep{elhage2021mathematical,olsson2022induction}. Whatever performs induction must therefore read and write positional signals. From the ten global bands identified above, we thus select as the induction-critical pair the two most position-specific bands, those with the highest ratio of positional coupling to token coupling, $\operatorname{PosRatio}(v_i) = C^2_{\mathrm{pos}}(v_i)/C^2_{\mathrm{tok}}(v_i)$. The numerator and denominator are band $v_i$'s squared \emph{positional coupling} with the positional matrix $\Pi$ and its squared \emph{token coupling} with the token embedding matrix $W_E$, $C^{2}_{\mathrm{pos}}(v_i) = \lVert \Pi^{\top} v_i \rVert^{2}/\lVert \Pi \rVert_F^{2}$ and $C^{2}_{\mathrm{tok}}(v_i) = \lVert W_E^{\top} v_i \rVert^{2}/\lVert W_E \rVert_F^{2}$. We divide $C^{2}_{\mathrm{pos}}$ by $C^{2}_{\mathrm{tok}}$ because the deletion is meant to be selective. The goal is to find the directions that are most positional and least content-carrying, because deleting content-carrying directions degrades prediction indiscriminately \citep{kovaleva2021bert,dettmers2022gpt3}. We construct the \emph{positional matrix} $\Pi$ differently for the two positional schemes, as follows.

\subsection{The positional matrix}
\label{app:pos}

Both positional schemes need a positional matrix $\Pi$, but they differ in whether that matrix can be read from the weights directly. GPT-2 has a learned positional embedding matrix $W_{\mathrm{pos}}$, which writes position into the stream at layer $0$ directly, so the positional matrix is the positional embedding matrix itself, $\Pi = W_{\mathrm{pos}}$.

By contrast, in transformers that rely on rotary embeddings \citep{su2024roformer}, as in the Pythia family, positions enter only through the attention logits, as $q_i^{\top} R(i - j)\, k_j$ for a rotation $R$ that depends on the relative position $i - j$, so there is no $W_{\mathrm{pos}}$ matrix to read. However, the stream still carries positional information. We therefore estimate the positional matrix from the residual stream's activations instead, and call the estimated positional matrix $\hat{\Pi}$. The model is fed $32$ sequences of $128$ tokens, each a BOS token followed by $127$ uniform-random tokens, so position is the only systematic variable. The column of $\hat{\Pi}$ for position $t$ is the mean residual-stream state at that position at the middle layer's input, centered over positions (Appendix~\ref{app:universality}, Eq.~\ref{eq:posprofile}). The same construction runs unchanged at every rotary scale, from Pythia-160m to 6.9B, with only the stream width $d$ and the middle-layer index changing with the architecture.

\subsection{Deletion protocol and results}

\paragraph{Induction-critical subspace deletion.}
Once the two-dimensional induction-critical subspace is identified, we delete it by projection in three steps. First, the two chosen directions are stacked as the columns of an orthonormal matrix $V \in \mathbb{R}^{d\times 2}$. Second, this subspace is removed from the residual stream at the input of every layer and before the final unembedding, by replacing the residual stream's activation $x$ by $x - VV^{\top}x$. Third, we measure the induction gain (Eq.~\ref{eq:gain}) of the transformer without this two-dimensional subspace, as well as the rise in natural-text loss on the Pile.

\begin{table}[t]
\centering
\caption{\textbf{The deletion protocol and subspace comparison at
six scales.} Each condition cell gives the percentage of the clean induction gain destroyed by deleting the two-dimensional subspace that the method selects, followed by the increase in natural-text loss on the Pile ($\Delta$NLL, in nats). The methods are RW-PCA (ours), activation PCA \citep{mu2018allbutthetop,ethayarajh2019contextual}, outlier dimensions \citep{timkey2021all,kovaleva2021bert,dettmers2022gpt3}, and positional PCA, the top two principal directions of the positional matrix. The last two columns are the two principal
cosines between the RW-PCA subspace and each comparison subspace ($1$ =
identical, $0$ = orthogonal). Random deletes a random two-dimensional subspace (median of five draws). GPT-Neo is excluded because random deletion already removes $60$--$98\%$ of its induction gain.}
\label{tab:cs2rep}
\footnotesize
\setlength{\tabcolsep}{2.5pt}
\begin{tabular}{lrrrrrcc}
\toprule
& & & & & & \multicolumn{2}{c}{Cosines: ours vs} \\
\cmidrule(lr){7-8}
Model & RW-PCA (ours) & Activation PCA & Outlier dims & Positional PCA & Random & Act.\ PCA & Outliers \\
\midrule
GPT-2 & $92.3\%$ / $+1.6$ & $87.4\%$ / $+7.0$ & $57.5\%$ / $+4.2$ & $31.3\%$ / $+0.4$ & $0.2\%$ / $+0.1$ & $.29,\,.07$ & $.08,\,.00$ \\
GPT-2-medium & $82.5\%$ / $+3.4$ & $88.5\%$ / $+7.7$ & $71.2\%$ / $+5.1$ & $2.3\%$ / $+0.1$ & $1.3\%$ / $+0.1$ & $.25,\,.01$ & $.16,\,.00$ \\
GPT-2-large & $81.8\%$ / $+1.3$ & $86.6\%$ / $+3.4$ & $-19.0\%$ / $+1.8$ & $2.0\%$ / $+0.0$ & $0.1\%$ / $+0.0$ & $.50,\,.05$ & $.05,\,.01$ \\
Pythia-160m (rotary) & $95.8\%$ / $+8.1$ & $42.0\%$ / $+3.8$ & $35.4\%$ / $+2.4$ & $32.1\%$ / $+4.2$ & $3.0\%$ / $+0.2$ & $.66,\,.12$ & $.59,\,.00$ \\
Pythia-2.8B (rotary) & $85.9\%$ / $+9.7$ & $78.0\%$ / $+7.2$ & $83.5\%$ / $+6.9$ & $87.9\%$ / $+6.5$ & $0.2\%$ / $+0.0$ & $.26,\,.02$ & $.21,\,.01$ \\
Pythia-6.9B (rotary) & $89.5\%$ / $+8.6$ & $84.5\%$ / $+8.2$ & $36.5\%$ / $+8.4$ & $68.1\%$ / $+9.3$ & $0.2\%$ / $+0.0$ & $.83,\,.26$ & $.84,\,.34$ \\
\bottomrule
\end{tabular}
\end{table}

\paragraph{Comparison subspaces.}
No existing method finds induction-critical directions of the residual stream from the weights. We therefore compare the induction-critical subspace derived from RW-PCA with two-dimensional subspaces derived from three other methods in the literature. One method, \emph{activation PCA}, takes the top two principal components of the mid-layer residual stream on natural text \citep{mu2018allbutthetop,ethayarajh2019contextual}. Another method, \emph{outlier dimensions}, gives the two stream coordinates with the largest activation RMS, dimensions $447$ and $138$ in GPT-2 \citep{timkey2021all,kovaleva2021bert,dettmers2022gpt3}. These two are the established unsupervised ways of picking out important residual stream directions. We test whether the subspace derived from RW-PCA is distinct from the subspaces these two prior methods derive, and which subspace carries more induction information. The third method, \emph{positional PCA}, takes the top two principal directions of the positional matrix, $W_{\mathrm{pos}}$ for the GPT-2 family and $\hat{\Pi}$ for the rotary models. Against positional PCA, we test whether the map's pair is simply the dominant positional directions. Lastly, we use random subspaces as the matched controls.

\paragraph{Distinct directions.}
Table~\ref{tab:cs2rep} (last two columns) reports the two principal cosines between the RW-PCA plane and each of the two unsupervised comparison planes on all six models. For orthonormal bases $V_a, V_b \in \mathbb{R}^{d \times 2}$ of two planes, the principal cosines are the two singular values of $V_a^{\top} V_b$ \citep{bjorck1973numerical}, so a pair of identical planes scores $1, 1$ and a pair of orthogonal planes scores $0, 0$. The two unsupervised subspaces nearly coincide with each other, with first cosines of $0.93$--$0.998$ on five of six models. The RW-PCA plane is distinct from both, because its second principal cosine against either never exceeds $0.35$ and is usually below $0.1$. The RW-PCA plane is also distinct from the positional plane. Its first principal cosine against positional PCA ranges from $0.16$ to $0.80$ across the six models, and its second never exceeds $0.23$ (Table~\ref{tab:poscos}). On every model, RW-PCA therefore identifies at least one induction-critical direction that none of the three comparison subspaces captures.

\paragraph{Deletion outcomes.}
Deleting the two RW-PCA directions destroys $82$--$96\%$ of the induction gain on every model (Table~\ref{tab:cs2rep}). The comparison subspaces are either less effective or inconsistent across scales. Positional PCA destroys only $2$--$32\%$ on four of the six models and $68\%$ on Pythia-6.9B, and it matches RW-PCA only on Pythia-2.8B ($87.9\%$ against $85.9\%$). The map's pair is therefore not simply the two dominant positional directions. On Pythia-6.9B, projecting out only $2$ of its $4{,}096$ dimensions, merely $0.05\%$ of the residual stream's width, abolishes $89.5\%$ of its induction capability, whereas activation PCA achieves $84.5\%$ and outlier dimensions only $36.5\%$. The contrast is starker on Pythia-160m, where RW-PCA destroys $95.8\%$ of the induction gain, but activation PCA destroys only $42\%$ and outlier dimensions only $35.4\%$. Random subspace deletions have zero or near-zero effect on all six models. Appendix~\ref{app:universality} separates the two block losses behind each gain and shows that under the RW-PCA deletion the second-block loss rises to the first-block level in every model, so the gain shrinks because copying stops. A deletion-dose check there on the three rotary models confirms that two dimensions is a minimal and effective choice.

\subsection{Robustness}
\label{app:rulerobustness}

Three checks support the band selection rule and the estimated positional matrix (Appendix~\ref{app:universality}). First, selecting the top two bands of Pythia-160m by eigenvalue alone, ignoring position, destroys only $46\%$ of the induction gain, against $95.8\%$ for the $\operatorname{PosRatio}$ selection, so the ratio does the work. Second, building $\hat{\Pi}$ from $4$ to $128$ random-token sequences at three seeds per size selects the identical pair every time. Third, GPT-Neo-125M is excluded from Table~\ref{tab:cs2rep} because random two-dimensional deletions already destroy $60$--$98\%$ of its induction gain, so no two-dimensional deletion is distinguishable from a random one on that model.

In summary, the map identifies a two-dimensional subspace, distinct from those of the prior literature, that carries a transformer's in-context copying capability at every scale tested. RW-PCA derives its bands from the weights alone, whereas activation PCA's components depend on the corpus and the layer sampled, and RW-PCA is the only one of the four methods whose deletion effect holds at every scale.

\section{Conclusion}
\label{sec:conclusion}

In this paper, we build a communication map of an entire transformer, covering every head, neuron, and eligible channel, from the geometry of its weight matrices, with exact and efficient algorithms that build the map of a $6.9$B-parameter model in minutes. The unit of analysis is the whole transformer rather than the individual head or circuit, so the map answers questions that no single-circuit analysis can pose. We find that the census sits at chance at initialization, that super-coupling appears within the first thousand steps, and that avoidance is learned later. Heads group into vertical communities, the induction heads and most of the IOI name-mover heads concentrate in one of them, and ablating that community destroys the model's induction capability. In the application, we identify two global communication bands that form an induction-critical subspace of the residual stream. These findings replicate across six models, from GPT-2 small to Pythia-6.9B.

Our research has five limitations. First, the map charts potential connectivity, not traffic. While our technique shows how components are geometrically coupled to communicate with each other in a transformer, it does not demonstrate that such coupled pairs actually communicate in forward passes on any given input. Measuring which channels carry traffic on a given input requires activation-level methods, which we leave to future work. Second, the validation and the application both rest on one capability, in-context copying, and what the other communities of the head graph do lies outside this paper's scope. Third, the interventions are head-level, except for the neuron ablation of Appendix~\ref{app:neuronwriters}, so the map's neuron-level structure remains largely untested. Fourth, our experiments stop at the sub-10B-parameter scale. Fifth, the map's nodes are heads and neurons, not learned features. A feature-level map built from sparse-autoencoder or transcoder units \citep{bricken2023monosemanticity,cunningham2024sparse,dunefsky2024transcoders} is a natural extension in future research.

\bibliography{references}
\bibliographystyle{tmlr}

\appendix
\clearpage

\section{Mathematical appendix}
\label{app:math}

\subsection{Conventions}
\label{app:conventions}
Stream vectors $x \in \mathbb{R}^{d}$ are columns, and all maps act by
left multiplication. All four per-head weight matrices are stored
uniformly as factors
$W_Q, W_K, W_V, W_O \in \mathbb{R}^{d \times d_{\mathrm{head}}}$.
Reading projects into head space through a transpose
($W_X^{\top} x \in \mathbb{R}^{d_{\mathrm{head}}}$ for
$X \in \{Q, K, V\}$), and writing returns through left multiplication
($W_O z \in \mathbb{R}^{d}$). Consequently $W_{QK} = W_Q W_K^{\top}$ and
$W_{OV} = W_O W_V^{\top}$ are $d \times d$ with rank at most
$d_{\mathrm{head}}$. The letters are mnemonic and mark a type
distinction. $W_{QK}$ is a bilinear \emph{form} (two reading bases, analyzed
by SVD, never applied as an operator), whereas $W_{OV}$ is a stream-to-stream
\emph{map} (analyzed by eigendecomposition). Common implementations, including TransformerLens
\citep{nanda2022transformerlens}, store activations as row vectors
with weights multiplying on the right, where $W_Q, W_K, W_V$ already
have the factor shape and $W_O$ must be transposed.
\citet{elhage2021mathematical} write the same operator as
$W_{OV} = W_O W_V$ with map-shaped factors, the identical object. Our
loading code performs the conversion once and asserts the resulting
$W_{QK}$ and $W_{OV}$ against TransformerLens's own factored QK and OV
circuits, so a silently transposed map is caught at load time. Table~\ref{tab:shapes} collects the
shapes, ranks, and roles of every per-component matrix in one place.

\begin{table}[h]
\centering
\caption{\textbf{Shapes, ranks, and roles of the per-component matrices}
(column convention). Ranks are generic, meaning that with trained weights all
$d_{\mathrm{head}}$ singular values are numerically nonzero, though
\emph{effective} ranks are often far smaller.}
\label{tab:shapes}
\footnotesize
\setlength{\tabcolsep}{1.5pt}
\begin{tabular}{lllll}
\toprule
Matrix & Shape (general) & Shape (GPT-2) & Rank & Role \\
\midrule
$W_Q$ & $d \times d_{\mathrm{head}}$ & $768 \times 64$ & $\le d_{\mathrm{head}}$ & query factor: $q = W_Q^{\top} x$ \\
$W_K$ & $d \times d_{\mathrm{head}}$ & $768 \times 64$ & $\le d_{\mathrm{head}}$ & key factor: $k = W_K^{\top} x$ \\
$W_V$ & $d \times d_{\mathrm{head}}$ & $768 \times 64$ & $\le d_{\mathrm{head}}$ & value factor: $v = W_V^{\top} x$ \\
$W_O$ & $d \times d_{\mathrm{head}}$ & $768 \times 64$ & $\le d_{\mathrm{head}}$ & output factor: head space $\to$ stream, $\Delta = W_O z$ \\
$W_{QK} = W_Q W_K^{\top}$ & $d \times d$ & $768 \times 768$ & $\le d_{\mathrm{head}}$ & QK circuit: bilinear form on position pairs \\
$W_{OV} = W_O W_V^{\top}$ & $d \times d$ & $768 \times 768$ & $\le d_{\mathrm{head}}$ & OV circuit: stream $\to$ stream operator \\
$r_m$ & $d \times 1$ & $768 \times 1$ & $1$ & neuron read direction: $a_m = r_m^{\top} x$ \\
$w_m$ & $d \times 1$ & $768 \times 1$ & $1$ & neuron write direction: $\Delta_m = g(a_m)\, w_m$ \\
\bottomrule
\end{tabular}
\end{table}

\subsection{The attention block, derived to the factored form}
\label{app:attnblock}
With positions stacked as columns of $X \in \mathbb{R}^{d \times n}$, the head's write is $\Delta X = W_{OV,h}\, X A^{\top}$. The head's action separates along two axes. The attention matrix $A$ acts only on the position axis, and it is recomputed on every forward pass. By contrast, $W_{QK,h}$ (through the logits) and $W_{OV,h}$ act only on the feature axis, and they are frozen weights. The map charts the feature axis (Section~\ref{sec:background}).

\paragraph{The four-term logit split behind K- and Q-composition.}
The reader's attention logit
$\ell_{ij} = x_i^{\top} W_{QK,r}\, x_j / \sqrt{d_{\mathrm{head}}}$
is bilinear in the two stream
vectors $x_i$ and $x_j$, and an upstream writer contributes to
both, since it writes at every position. Let
$W_{OV,w} u_i$ denote the writer's output at position $i$, with $u_i$ its
attention-averaged input there, and split $x_i = \bar{x}_i + W_{OV,w} u_i$
into the writer's write and the rest of the stream (same for $x_j$).
Bilinearity gives, exactly,
\begin{equation}
\begin{split}
\sqrt{d_{\mathrm{head}}}\;\ell_{ij} \;=\;{}&
  \bar{x}_i^{\top} W_{QK,r}\, \bar{x}_j
  \;+\; \bar{x}_i^{\top} W_{QK,r}\, (W_{OV,w} u_j) \\
  &+ (W_{OV,w} u_i)^{\top} W_{QK,r}\, \bar{x}_j
  \;+\; (W_{OV,w} u_i)^{\top} W_{QK,r}\, (W_{OV,w} u_j).
\end{split}
\label{eq:foursplit}
\end{equation}
The first term does not involve the writer. The second is the writer's
entry through the key side, governed by the operator
$W_{QK,r} W_{OV,w}$ (K-composition). The third is the writer's entry
through the query side, governed by
$W_{QK,r}^{\top} W_{OV,w}$ (Q-composition). The two cross terms are not exclusive
cases but coexisting channels of one exact expansion. The fourth term
routes the writer into both sides at once, through the operator
$W_{OV,w}^{\top} W_{QK,r} W_{OV,w}$.

The both-sides term adds no channel of its own. Its operator
contains the writer's matrix twice, so its normalized score
carries $\lVert W_{OV,w} \rVert_F^{2}$ in the denominator. The
Frobenius norm is submultiplicative,
$\lVert XY \rVert_F \le \lVert X \rVert_F\, \lVert Y \rVert_F$
\citep{horn1991topics}, and applying it to the grouping
$W_{OV,w}^{\top}(W_{QK,r} W_{OV,w})$ gives
\begin{equation}
\frac{\lVert W_{OV,w}^{\top} W_{QK,r}\, W_{OV,w} \rVert_F}
{\lVert W_{QK,r} \rVert_F\, \lVert W_{OV,w} \rVert_F^{2}}
\le
\frac{\lVert W_{OV,w} \rVert_F\, \lVert W_{QK,r} W_{OV,w} \rVert_F}
{\lVert W_{QK,r} \rVert_F\, \lVert W_{OV,w} \rVert_F^{2}}
=
\frac{\lVert W_{QK,r} W_{OV,w} \rVert_F}
{\lVert W_{QK,r} \rVert_F\, \lVert W_{OV,w} \rVert_F}
= C_K,
\end{equation}
after one factor of $\lVert W_{OV,w} \rVert_F$ cancels, so the
both-sides score is bounded by exactly the K-composition
coupling. The other grouping,
$(W_{OV,w}^{\top} W_{QK,r})\, W_{OV,w}$, gives the mirror bound by
the Q-composition coupling $C_Q$. A vanishing
one-sided channel therefore forces the both-sides term to zero,
so the map scores two QK edge classes per head pair rather than
three. Numerically, the bound holds with a wide
margin. Across all $9{,}504$ head pairs of GPT-2 the normalized
both-sides score never exceeds its bound, and its ratio to the
smaller of the two one-sided couplings, $\min(C_K, C_Q)$, is at
most $0.35$ (script released).

\subsection{The coupling coefficient}
\label{app:statistic}

\begin{proposition}[the weighted-cosine identity]
\label{prop:factor}
Let $R = U \Sigma V^{\top}$ and $W = A S B^{\top}$ be the two
SVDs, with singular values $\sigma_k$ and $s_\ell$, and let
$\cos\theta_{k\ell} := v_k^{\top} a_\ell$ pair the reader's input
directions with the writer's output directions. With the gain
shares $p_k := \sigma_k^2 / \sum_{k'} \sigma_{k'}^2$ and
$q_\ell := s_\ell^2 / \sum_{\ell'} s_{\ell'}^2$,
\begin{equation}
C^2 \;=\; \sum_{k,\ell} p_k\, q_\ell\, \cos^2\theta_{k\ell}.
\nonumber
\end{equation}
\end{proposition}

\begin{proof}
The composed operator factorizes as
$RW = U\, \Sigma\, (V^{\top} A)\, S\, B^{\top} =
U\, (\Sigma \Theta S)\, B^{\top}$ with the cosine table
$\Theta := V^{\top} A$, $\Theta_{k\ell} = \cos\theta_{k\ell}$.
The outer factors have orthonormal columns and preserve the
Frobenius norm, since
$\operatorname{tr}(B X^{\top} U^{\top} U X B^{\top}) =
\operatorname{tr}(X^{\top} X)$ by cyclicity, so
$\lVert RW \rVert_F^{2} = \lVert \Sigma \Theta S \rVert_F^{2}$.
The left diagonal factor $\Sigma$ scales the rows of $\Theta$,
and the right diagonal factor $S$ scales the columns of
$\Theta$. Thus
$(\Sigma \Theta S)_{k\ell} = \sigma_k \cos\theta_{k\ell}\,
s_\ell$, and summing squared entries gives
$\lVert RW \rVert_F^{2} = \sum_{k,\ell} \sigma_k^2 s_\ell^2
\cos^2\theta_{k\ell}$. The same norm preservation applied to each
SVD alone gives the denominators,
$\lVert R \rVert_F^{2} = \lVert \Sigma \rVert_F^{2} =
\sum_k \sigma_k^2$ and
$\lVert W \rVert_F^{2} = \lVert S \rVert_F^{2} =
\sum_\ell s_\ell^2$, and dividing the
squared numerator by both gives Eq.~\ref{eq:cosform}.
\end{proof}

\begin{proposition}[Gram form]
\label{prop:gram}
\begin{equation}
C^{2}
\;=\; \frac{\lVert RW \rVert_F^{2}}
{\lVert R \rVert_F^{2}\, \lVert W \rVert_F^{2}}
\;=\; \frac{\operatorname{tr}(GH)}
{\operatorname{tr}(G)\, \operatorname{tr}(H)},
\qquad G := R^{\top} R,
\quad H := W W^{\top}.
\nonumber
\end{equation}
\end{proposition}

\begin{proof}
By cyclicity of the trace,
\begin{equation}
\lVert RW \rVert_F^2
= \operatorname{tr}(W^{\top} R^{\top} R W)
= \operatorname{tr}\big(\underbrace{R^{\top} R}_{G}\,
\underbrace{W W^{\top}}_{H}\big),
\nonumber
\end{equation}
and the normalizers are
$\operatorname{tr} G = \operatorname{tr}(R^{\top}R) =
\lVert R \rVert_F^{2}$ and likewise
$\operatorname{tr} H = \lVert W \rVert_F^{2}$.
\end{proof}

The numerator is the Frobenius inner product of the two
Gram matrices, since
$\operatorname{tr}(GH) = \operatorname{tr}(G^{\top}H)
= \langle G, H \rangle_F$ by the symmetry of $G$. The
denominator is a choice. Normalizing by the traces, rather than
by Cauchy--Schwarz with
$\lVert G \rVert_F\, \lVert H \rVert_F$, is what yields the
weighted-cosine reading of Eq.~\ref{eq:cosform} and the
universal chance level of Appendix~\ref{app:nulls}. The two normalizations give opposite verdicts on one test case, a perfectly matched isotropic pair, $R = W = I_d$. Under the trace normalization that pair scores $C^2 = 1/d$, exactly the chance level of Proposition~\ref{prop:chance}. By contrast, the Cauchy--Schwarz normalization would award the same pair its maximal score. The
trace verdict is the intended one, because a pair coupled equally
over every direction of the stream exhibits none of the selective
alignment that defines an edge (Section~\ref{sec:background}).

\begin{corollary}[invariance]
\label{cor:invariance}
$C$ depends on the pair only through $(G, H)$. It satisfies
$C(\alpha R, \beta W) = C(R, W)$ for any nonzero scalars and
$C(UR,\, WV) = C(R, W)$ for any orthogonal $U$, $V$.
\end{corollary}

\begin{proof}
$(UR)^{\top}(UR) = G$ and $(WV)(WV)^{\top} = H$, and scalars
cancel in the ratio.
\end{proof}

These invariances are the right ones for a communication statistic. $U$ rotates the reader's internal encoding after
reading, and $V$ rotates the writer's internal channels. Neither
is visible on the residual stream, so neither affects an edge
score. Only the listened-to directions ($G$) and the written-onto
directions ($H$) survive.

\begin{corollary}[range and extremes]
\label{cor:bounds}
$C \in [0, 1]$. Moreover: (i) $C = 0$ if and only if $RW = 0$,
equivalently $\mathrm{Row}(R) \perp \mathrm{Col}(W)$; (ii) $C = 1$
if and only if $R$ and $W$ both have rank one and the reader's
input direction coincides with the writer's output direction up to
sign; (iii) for rank-one pairs,
$C = \lvert\cos\angle(r_m, w_{m'})\rvert$.
\end{corollary}

\begin{proof}
By Proposition~\ref{prop:factor},
$C^2 = \sum_{k,\ell} p_k\, q_\ell \cos^2\theta_{k\ell}$ is a
convex combination of the values $\cos^2\theta_{k\ell}$, since the
weights $p_k q_\ell$ are nonnegative and sum to one, so
$C \in [0, 1]$. Part (i) is immediate from the definition, because
$\lVert RW \rVert_F = 0$ if and only if $RW = 0$. For (ii), the
convex combination equals $1$ if and only if
$\cos^2\theta_{k\ell} = 1$ for every pair with $p_k q_\ell > 0$.
If $p$ had two nonzero entries $k \ne k'$, the directions $v_k$
and $v_{k'}$ would both coincide with the same unit vector
$a_\ell$ up to sign, contradicting their orthogonality. Hence $p$
and $q$ each have a single nonzero entry, both matrices have rank
one, and the two surviving directions satisfy
$\cos^2\theta = 1$, meaning they coincide up to sign. For (iii),
the SVD of a vector has a single singular direction, so the sum
has one term with $p_1 = q_1 = 1$.
\end{proof}

At $C = 0$ the reader receives nothing the writer transmits, and the
rank-one case is what puts neuron wires and head circuits on one
common scale.

\paragraph{Why the K-composition reader is $W_{QK,r}$, not $W_K$.}
Eq.~\ref{eq:cosform} also justifies the choice of reading matrix. The
key-side reading directions are the right singular vectors $v_k$ of
$W_{QK,r}$, and each is weighted by $\sigma_k$, which measures how strongly
the paired query direction $u_k$ listens. A key channel the query side
never matches has $\sigma_k \approx 0$ and contributes nothing to logits,
and Eq.~\ref{eq:statistic} counts it as nothing. A reader built from
$W_K$ alone would count it fully.

\subsection{Nulls}
\label{app:nulls}

The rotation null distribution of the main text is, in the
terminology of \citet{efron2004large}, the \emph{theoretical null
distribution}. The \emph{empirical null distribution}, the standard by which edges are selected, is in Section~\ref{app:selection}.

\begin{proposition}[universal chance level]
\label{prop:chance}
Under a Haar rotation of either side, $\mathbb{E}[C^2] = 1/d$
exactly, for any reader and writer singular values.
\end{proposition}

\begin{proof}
Writing $\hat{H} = H/\operatorname{tr}(H)$, invariance of the Haar
measure forces $\mathbb{E}_Q[Q\hat{H}Q^{\top}] = I/d$, a classical
averaging identity (\citealp[Cor.~3.4]{collins2006integration};
\citealp[\S 2.1]{meckes2019random}),
so, by linearity of the trace and of expectation,
$\mathbb{E}[C^2] =
\operatorname{tr}(\hat{G}\,\mathbb{E}_Q[Q\hat{H}Q^{\top}]) = 1/d$.
\end{proof}

Proposition~\ref{prop:chance} pins down the mean of the null
distribution and nothing more. How $C^2$ scatters around $1/d$
from one rotation to the next still depends on the singular
values. A writer whose singular values are all equal has
$\hat{H} = I/d$, a
matrix that every rotation leaves unchanged, so $C^2$ equals
$1/d$ with no fluctuation at all. A rank-one pair is the opposite
extreme. All of the gain lies in one band on each side, so the statistic is the squared cosine between a fixed direction and a uniformly random one. In this case the whole null distribution, not only its mean, has a classical closed form.

\begin{lemma}[{rank-one null law; classical, see \citealp[\S 2.1, Prop.~2.5]{meckes2019random}}]
\label{lem:beta}
For a fixed unit vector $u$ and $v$ uniform on the unit sphere in
$\mathbb{R}^d$, the squared cosine $(u^{\top} v)^2$ follows
$\mathrm{Beta}\big(\tfrac12, \tfrac{d-1}{2}\big)$, and under this
distribution the standard deviation of the cosine is
$1/\sqrt{d} = 0.036$ at $d = 768$.
\end{lemma}

\begin{corollary}[chance-law exceedance counts]
\label{cor:ceiling}
By Lemma~\ref{lem:beta} the two-sided tail is
$P(|\cos| > t) = I_{1-t^2}\!\big(\tfrac{d-1}{2}, \tfrac12\big)$,
the regularized incomplete beta function. Among the
$N = 622{,}854{,}144$ neuron-pair draws from the null distribution,
the expected count $N \cdot P(|\cos| > t)$ of pairs beyond
$t = 0.2$ is therefore $13.8$, beyond $0.23$ it is $0.07$, and
beyond the wire threshold $0.5$ it is $4\times10^{-41}$. The
maximum of the $N$ draws consequently has median $0.217$ and 95th
percentile $0.231$, solving $N \cdot P(|\cos| > t) = \ln 2$ and
$-\ln 0.95$ respectively.
\end{corollary}

The wire tail persists at every scale. At
$6.9$B the largest observed $|\cos|$ is $0.78$ against a $0.016$
chance scale. On natural text the wired pairs co-activate at median
$|\mathrm{corr}| = 0.31$ against $0.03$ for random pairs, so they
are not dormant. Table~\ref{tab:nnexceed} sets the observed
exceedance counts against the chance law at each threshold.

\begin{figure}[t]
\centering
\includegraphics[width=0.92\linewidth]{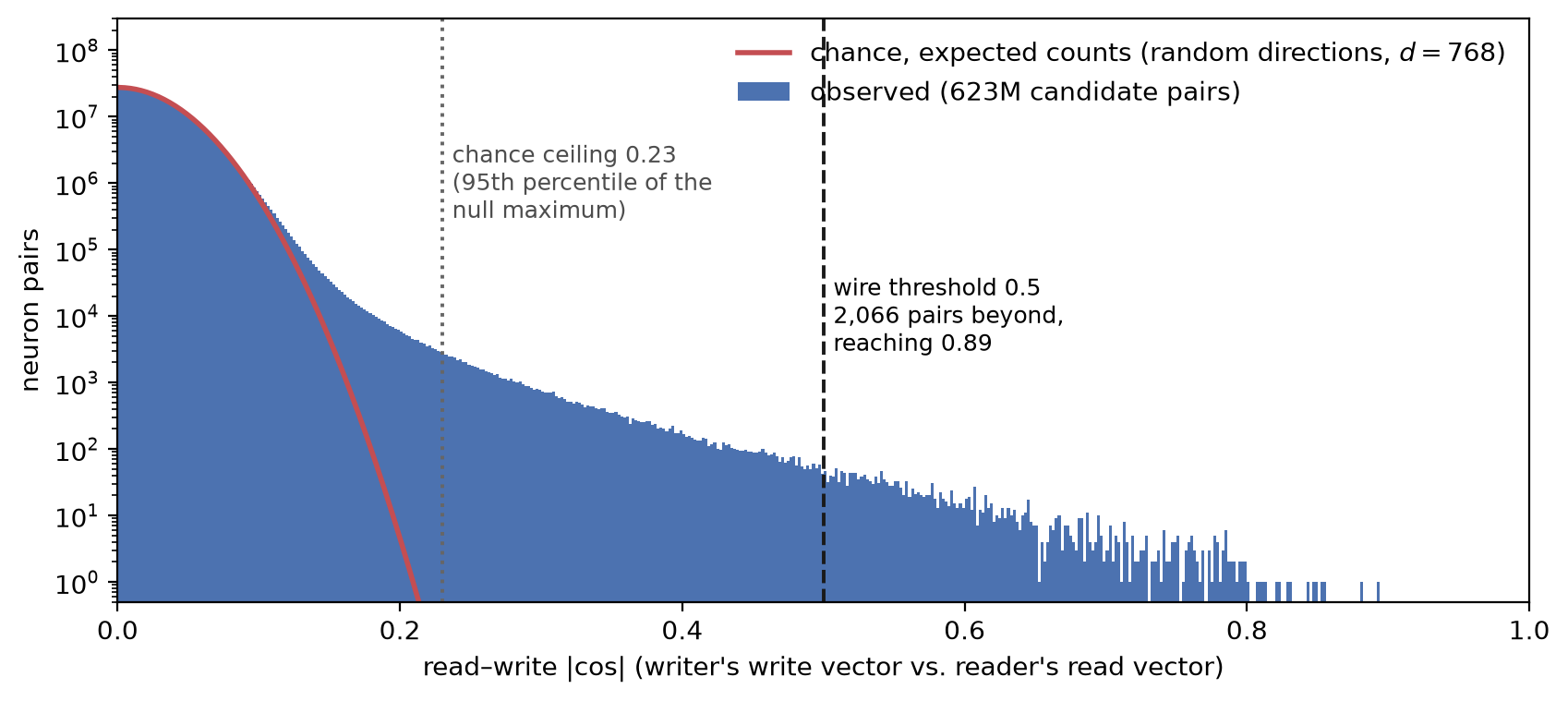}
\caption{\textbf{The neuron$\to$neuron coupling distribution.}
Read--write $|\cos|$ over all $622{,}854{,}144$ candidate neuron
pairs (blue, log count scale) against the exact chance density for
random directions in $d=768$ (red), scaled to the same number of
pairs. Chance and observation agree through the bulk. The dotted
line marks $0.23$, beyond which the chance law expects $0.07$
pairs. The observed distribution continues to $0.89$. The $2{,}066$ pairs beyond
$0.5$ are the \emph{wires}.}
\label{fig:nncos}
\end{figure}

\begin{figure}[t]
\centering
\includegraphics[width=0.92\linewidth]{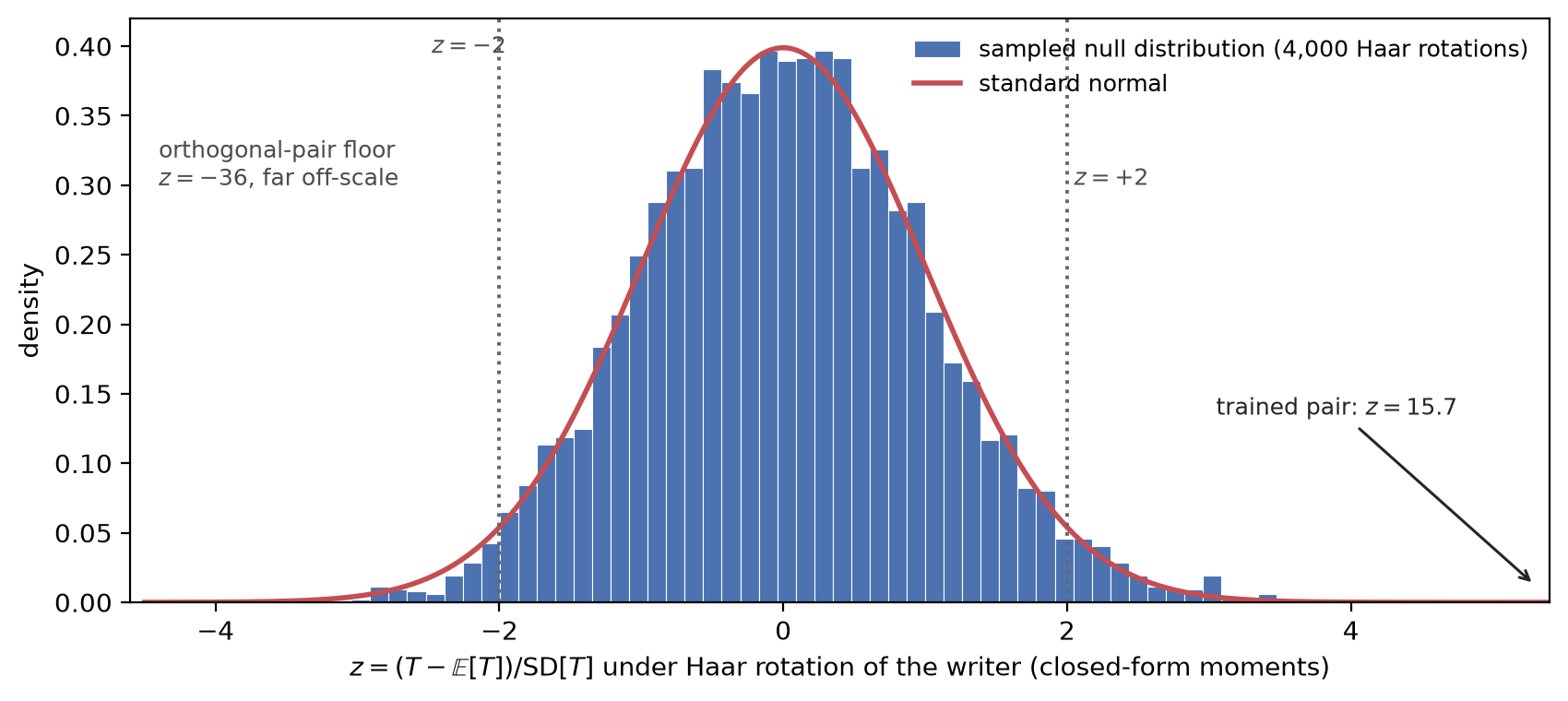}
\caption{The sampled rotation null distribution of one head pair
(writer L4H7 $\to$ reader L8H10 of GPT-2, $4{,}000$ Haar rotations),
standardized by the closed-form moments of
Eq.~\ref{eq:zcensus}. The rank-$64$ null is a near-symmetric
bump with mass in both tails (sampled mean $0.02$, SD $1.00$, skew
$+0.20$), so the below-chance regime of Table~\ref{tab:theorycensus} is reachable. The
perfectly orthogonal pair sits at $z=-36$, far off-scale, in contrast
with the rank-one law of Figure~\ref{fig:nncos}, whose floor of $-0.7$ makes
$z\le-2$ unreachable for neuron pairs. The trained pair scores
$z=15.7$, far outside its own null distribution.}
\label{fig:nullshape}
\end{figure}

\begin{table}[h]
\centering
\caption{\textbf{Exceedance counts over the neuron$\to$neuron
chance law (GPT-2).} For each threshold $t$: the single-pair tail
$P(|\cos| > t)$ under the exact chance law (Lemma~\ref{lem:beta}),
the resulting expected number of pairs beyond $t$ among all
$N = 622{,}854{,}144$ candidate pairs, and the observed count in
the trained weights.}
\label{tab:nnexceed}
\footnotesize
\begin{tabular}{cccr}
\toprule
$t$ & $P(|\cos| > t)$ & Expected $N \cdot P$ & Observed \\
\midrule
$0.15$ & $3.0\times10^{-5}$  & $1.8\times10^{4}$   & $511{,}605$ \\
$0.20$ & $2.2\times10^{-8}$  & $13.8$              & $140{,}927$ \\
$0.23$ & $1.1\times10^{-10}$ & $0.07$              & $78{,}861$ \\
$0.25$ & $2.0\times10^{-12}$ & $1.3\times10^{-3}$  & $55{,}991$ \\
$0.30$ & $1.9\times10^{-17}$ & $1.2\times10^{-8}$  & $26{,}100$ \\
$0.50$ & $7.0\times10^{-50}$ & $4.4\times10^{-41}$ & $2{,}066$ \\
\bottomrule
\end{tabular}
\end{table}

\paragraph{Closed-form moments and the census.}
Lemma~\ref{lem:beta} covers the rank-one classes. A subspace pair
of rank above one has no closed-form null law, because the
rotation null distribution of $C$ depends on the singular values
of both sides. Its first two moments are nevertheless closed-form,
and the $z$-scores of the census need nothing more. Fix one head
pair as trained, with reader Gram $G$ and writer Gram $H$. Under
the rotation null the squared statistic is
$C^{2} = T/(\operatorname{tr}G\,\operatorname{tr}H)$ with
$T = \operatorname{tr}(G\,QHQ^{\top})$ and $Q$ Haar-random. The
mean $\mathbb{E}[C^{2}] = 1/d$ is Proposition~\ref{prop:chance}.
The variance of $T$ is the degree-two case of the same Weingarten
calculus \citep[Cor.~3.4]{collins2006integration}:
\begin{equation}
\operatorname{Var}[T] \;=\; \frac{2}{(d-1)(d+2)}
\Big(\operatorname{tr} G^{2} -
\tfrac{(\operatorname{tr} G)^{2}}{d}\Big)
\Big(\operatorname{tr} H^{2} -
\tfrac{(\operatorname{tr} H)^{2}}{d}\Big),
\label{eq:theovar}
\end{equation}
and dividing by the constant
$(\operatorname{tr}G\,\operatorname{tr}H)^{2}$ turns it into
$\operatorname{Var}[C^{2}]$. Thus, we compute the standardized
score ($z$-score) for $C^{2}$ of each writer--reader pair as
\begin{equation}
z \;=\; \frac{C^{2} - 1/d}
{\sqrt{\operatorname{Var}[T]}\,\big/\,
(\operatorname{tr}G\,\operatorname{tr}H)}
\;=\;
\frac{T - \operatorname{tr}G\,\operatorname{tr}H/d}
{\sqrt{\operatorname{Var}[T]}},
\label{eq:zcensus}
\end{equation}
with both moments exact, so no approximation enters. The second
form is the one the released code computes. All four invariants
collapse to inner-Gram products,
$\operatorname{tr} G = \operatorname{tr}(g_Q g_K)$ and
$\operatorname{tr} G^{2} = \operatorname{tr}((g_Q g_K)^{2})$ for a
K-composition reader and likewise with $g_O g_V$ for V-composition
and for writers, so every pair's $z$-score costs a few
$d_{\mathrm{head}}$-wide products at any model size. On GPT-2 the
closed-form census reproduces a $500$-rotation Monte Carlo census
to a tenth of a percentage point on every fraction (released verification script). Table~\ref{tab:theorycensus} (Section~\ref{sec:map}) is
the resulting census. Pairs whose empirical $z$-score lies below the negative of the selection threshold, the avoidant counterpart of a selected edge, concentrate almost exclusively in the V channel ($28$ on GPT-2, $5{,}229$ at $2.8$B, $4{,}179$ at
$6.9$B).

The shares in Table~\ref{tab:theorycensus} use only these exact moments and assume nothing about the shape of the null distribution. The two kinds of class are nevertheless reported differently, because their null distributions have different shapes. For the head classes the sampled null is in practice close to Gaussian (Figure~\ref{fig:nullshape}), as the many independent subspace overlaps aggregate, so $z$-tails are the natural report. By contrast, the rank-one classes follow the Beta law of Lemma~\ref{lem:beta}, so the neuron census reports exceedance counts rather than $z$-tails.

\paragraph{Shared-rotation null distributions for the
interface-matrix classes.}
Interface classes violate the sparsity assumption behind the
empirical null distribution (Section~\ref{app:selection}), because nearly
every head genuinely reads the embedding, so their signal is dense
rather than sparse. They are scored against shared-rotation null distributions. The interface matrix's Gram $H$ is conjugated by Haar-random rotations, $H \to Q H Q^{\top}$ with $Q \in O(d)$, the statistic is recomputed for each rotation, and each edge is scored as a $z$-score against the mean and standard deviation of that ensemble. The rotations are shared across all readers of an
interface matrix, so the cost is a few hundred $d \times d$
conjugations rather than per-edge Monte Carlo. The ensemble is the
isotropic reference, which is anti-conservative under anisotropy,
so we score these classes as effect sizes rather than as sparse discoveries, and we select no edges from them.

\paragraph{Monte Carlo confirmation.}
We verify both closed forms by sampling the null distribution directly (scripts released). For a pair with Grams $G$ and $H$,
Haar-random rotations $Q_1, \ldots, Q_n \in O(d)$ give the
ensemble
\begin{equation}
C_i \;=\;
\sqrt{\frac{\operatorname{tr}\!\big(G\, Q_i H Q_i^{\top}\big)}
{\operatorname{tr}(G)\,\operatorname{tr}(H)}},
\qquad i = 1, \ldots, n,
\label{eq:isoensemble}
\end{equation}
which keeps both sides' singular values and randomizes only the
relative orientation of the read and write subspaces, so
$\{C_i\}$ samples the pair's theoretical null distribution
exactly. The ensemble's standard deviation, written
$s_{\mathrm{theory}}$, is the spread that chance alignment would
produce for the pair. Across
$20$ combinations of reader and writer singular-value profiles at $d = 768$, sampled Haar rotations give $\mathbb{E}[C^2] = 1/d$ within sampling error ($|z| \le 1.6$). Profiles with equal singular values attain $1/d$ exactly with zero variance, confirming that the singular values control only the fluctuations. The
rank-one cosine law matches
$\mathrm{Beta}(\tfrac12, \tfrac{d-1}{2})$ by Kolmogorov--Smirnov at
$d = 3, 8, 768$ ($p = 0.55$, $0.71$, $0.71$), with the first four
moments correct to four decimals.

\section{Computing the map: algorithms and reproducibility}
\label{app:pipeline}

\paragraph{Reproducibility statement.} All models are public checkpoints accessed via TransformerLens \citep{nanda2022transformerlens}, and natural-text corpora are fixed Pile samples \citep{gao2020pile}. The full pipeline is \pipelineloc\ with fixed seeds, and every number regenerates from the released code and map tables. The rest of this appendix lists decisions, corpora, and compute. All experiments run on a single $32$\,GB consumer GPU, and every number except the $6.9$B rows reproduces on smaller cards.

\subsection{Weight preprocessing: LayerNorm folding}
\label{app:lnfold}

With layer normalization $\mathrm{LN}(\cdot)$, a transformer's sublayers read the transformed activation $\mathrm{LN}(x)$ from the residual stream, not the raw activation $x$.
LayerNorm decomposes into a fixed linear map followed by an input-dependent
scalar and a constant shift,
\begin{equation}
\mathrm{LN}(x) \;=\; \frac{1}{\sigma(x)}\,
\operatorname{diag}(\gamma)\Big(I - \tfrac1d
\mathbf{1}\mathbf{1}^{\top}\Big)\, x \;+\; \beta,
\end{equation}
where $\sigma(x)$ is the standard deviation of the coordinates of
$x$ at that position and $\mathbf{1}$ is the all-ones vector. Both
pieces of the linear factor are input-independent, so they fold
into the weights. To account for LayerNorm, every matrix $W$ that reads the stream through this LayerNorm is replaced by
\begin{equation}
\widetilde{W} \;=\;
\Big(I - \tfrac1d \mathbf{1}\mathbf{1}^{\top}\Big)
\operatorname{diag}(\gamma)\, W,
\qquad\text{so that}\qquad
W^{\top}\mathrm{LN}(x) \;=\;
\frac{1}{\sigma(x)}\,\widetilde{W}^{\top}x \;+\; W^{\top}\beta.
\label{eq:lnfold}
\end{equation}
The rescaling $\operatorname{diag}(\gamma)$ and the centering
projector are thus absorbed into every reading matrix, and the
constant $W^{\top}\beta$ folds into the sublayer's bias, which the
map's statistics never touch. After folding, the only input
dependence separating a read from the raw stream is the scalar
gain $1/\sigma(x)$.

Writing matrices receive the complementary operation. Every write into the stream is centered along the stream dimension, removing its component in the $\mathbf{1}$ direction.
This changes nothing about any coupling, because the projector now
inside every reader annihilates that component anyway. Its purpose is to make the stored weights faithful to what the stream can carry. A writer's matrix then
contains only what some reader could ever receive, so the norms
and couplings the map computes from writers measure transmissible
signal and no dead mass. All couplings in the map are computed
exactly on these folded and centered weights.

What folding cannot remove is the gain $1/\sigma(x)$, one scalar
per position multiplying every coordinate of a read. Two consequences follow. First, a common scalar rescales all of a reader's incoming couplings equally, so the ranking of edges into a fixed reader is unaffected by it. Second, the overall strength of each write--read hop does vary with input, so path strengths chained across layers carry a per-hop multiplicative uncertainty set by the spread of $\sigma(x)$.

\subsection{From $C^2$ to algorithms}
\label{app:guide}

The statistic of Eq.~\ref{eq:statistic} becomes the released algorithms in four steps. Throughout, a \emph{connection class} means one type of reader paired with
one type of writer, for example head $\to$ neuron, every eligible
pair of that type scored and calibrated together
(Table~\ref{tab:census} lists all $18$). Table~\ref{tab:rw} below tabulates every reader and writer with its Gram and its normalizer, and the next three subsections each cover one shape of pair: both sides full matrices (Appendix~\ref{app:gemm}), one side a single vector (Appendix~\ref{app:mixed}), and both sides single vectors (Appendix~\ref{app:stream}). These subsections introduce no new mathematics, only the definition of $C$ unwound by trace cyclicity
and applied at scale.

\paragraph{Ingredients.}
Every score in the map is assembled from a small set of
precomputed ingredients. Each head contributes four factors
$W_Q, W_K, W_V, W_O \in \R^{d \times d_{\mathrm{head}}}$
($768 \times 64$ on GPT-2 small). From these we first
precompute and cache the four
$d_{\mathrm{head}} \times d_{\mathrm{head}}$ \emph{inner Grams}
per head,
$g_Q := W_Q^{\top} W_Q$, $g_K := W_K^{\top} W_K$,
$g_V := W_V^{\top} W_V$, $g_O := W_O^{\top} W_O$. The inner Grams
are the reusable blocks from which the full Grams, the
normalizers, and the mixed-class quadratic forms are all
assembled. The factors also compose into the
two per-head circuits, the QK circuit $W_{QK} := W_Q W_K^{\top}$,
which scores query-key pairs, and the OV circuit
$W_{OV} := W_O W_V^{\top}$, which is the head's write operator. The
interface matrices contribute the embedding and positional writers $W_E$
and $W_{\mathrm{pos}}$ and the unembedding reader $W_U$. Each
neuron $m$ contributes its read and write vectors
$r_m, w_m \in \R^{d}$, used unit-normalized as
$\hat{r}_m, \hat{w}_m$ throughout.

\begin{table}[htbp]
\centering
\caption{\textbf{Every reader and writer of the map}, as
instantiations of $G = R^{\top} R$ and $H = W W^{\top}$. Only the
three interface-matrix Grams are built as arrays (the
rotation-null sampling conjugates them). A dash marks a Gram never
formed, collapsed through the inner-Gram identities of
Appendix~\ref{app:gemm}. Neuron normalizers are folded into the
unit vectors up front.}
\label{tab:rw}
\footnotesize
\setlength{\tabcolsep}{3pt}
\begin{tabular}{llllc}
\toprule
Component & Matrix & Gram & Normalizer & Rank \\
\midrule
\multicolumn{5}{l}{\emph{Head, reader side (one row per channel)}} \\
\quad K-composition & $R = W_{QK,r}$ & --- &
  $\lVert W_{QK} \rVert_F^2 = \operatorname{tr}(g_Q g_K)$ & $\le d_{\mathrm{head}}$ \\
\quad Q-composition & $R = W_{QK,r}^{\top}$ & --- &
  same as K-composition & $\le d_{\mathrm{head}}$ \\
\quad V-composition & $R = W_{OV,r}$ & --- &
  $\lVert W_{OV} \rVert_F^2 = \operatorname{tr}(g_O g_V)$ & $\le d_{\mathrm{head}}$ \\
\midrule
\multicolumn{5}{l}{\emph{Head, writer side (one row serves all three channels)}} \\
\quad head write & $W = W_{OV,w}$ & --- &
  $\lVert W_{OV} \rVert_F^2 = \operatorname{tr}(g_O g_V)$ & $\le d_{\mathrm{head}}$ \\
\midrule
\multicolumn{5}{l}{\emph{Neurons}} \\
\quad neuron read & $R = r_m^{\top}$ & --- &
  $\lVert r_m \rVert_2^2 = r_m^{\top} r_m$, folded into $\hat{r}_m$ & $1$ \\
\quad neuron write & $W = w_m$ & --- &
  $\lVert w_m \rVert_2^2 = w_m^{\top} w_m$, folded into $\hat{w}_m$ & $1$ \\
\midrule
\multicolumn{5}{l}{\emph{Interfaces}} \\
\quad embedding write & $W = W_E$ & $H = W_E W_E^{\top}$ &
  $\lVert W_E \rVert_F^2 = \operatorname{tr} H$ & $\le d$ \\
\quad positional write & $W = W_{\mathrm{pos}}$ & $H = W_{\mathrm{pos}} W_{\mathrm{pos}}^{\top}$ &
  $\lVert W_{\mathrm{pos}} \rVert_F^2 = \operatorname{tr} H$ & $\le d$ \\
\quad unembedding read & $R = W_U^{\top}$ & $G = W_U W_U^{\top}$ &
  $\lVert W_U \rVert_F^2 = \operatorname{tr} G$ & $\le d$ \\
\bottomrule
\end{tabular}
\end{table}

\paragraph{Assembling an edge.}
Every edge in the map is scored the same way. Take one reader row of Table~\ref{tab:rw} and one writer row, check that the pair is causally eligible (the masks are listed in Table~\ref{tab:census}), and evaluate $C^2 = \operatorname{tr}(GH)/(\operatorname{tr} G \cdot \operatorname{tr} H)$ with that reader's $G$ and that writer's $H$, through the collapsed forms of Appendix~\ref{app:gemm}.
The numerators of Table~\ref{tab:channels} are the gain-weighting
of Proposition~\ref{prop:factor} in computational form. The
K-composition numerator weights the cross factor by the
\emph{query} inner Gram $g_Q$, even though K-composition reads
through the key factor, a deliberate reversal discussed in
Appendix~\ref{app:statistic}. Two normalizers serve all four head rows,
$\operatorname{tr}(g_Q g_K) = \lVert W_{QK} \rVert_F^2$ for the K-
and Q-composition readers and $\operatorname{tr}(g_O g_V) = \lVert
W_{OV} \rVert_F^2$ for both the V-composition reader and the
writer.

\subsection{Full-rank pairs, factored traces through the inner Grams}
\label{app:gemm}

When both the reader and the writer are full matrices, as in the head $\to$ head channels and the interface-matrix $\leftrightarrow$ head pairs, we compute $C$ as follows. In the case of K-composition, the numerator
$\lVert RW \rVert_F^{2} = \lVert W_{QK,r} W_{OV,w} \rVert_F^{2}$
is needed for
every causally eligible pair, $9{,}504$ per channel on GPT-2.
Computing it naively would require multiplying two $d \times d$
matrices per pair, $d^{3} \approx 4.5 \times 10^{8}$ multiply-adds each at
GPT-2's width and growing cubically with the stream, about
$10^{13}$ operations for the head-to-head classes alone. We do not approximate this number. Instead, we evaluate exactly the same number using only $d_{\mathrm{head}}$-wide objects. The largest per-pair cost then becomes the $d\, d_{\mathrm{head}}^{2}$ multiplication that builds the cross factor, smaller by a factor of $(d / d_{\mathrm{head}})^{2}$, which is $144$ on GPT-2 and $1{,}024$ on Pythia-6.9B.

Substitute the factored forms
$W_{QK,r} = W_{Q,r} W_{K,r}^{\top}$ and
$W_{OV,w} = W_{O,w} W_{V,w}^{\top}$ into
$\lVert M \rVert_F^{2} = \operatorname{tr}(M^{\top} M)$:
\begin{align}
\lVert W_{QK,r} W_{OV,w} \rVert_F^{2}
&= \operatorname{tr}\big(W_{OV,w}^{\top}\, W_{QK,r}^{\top}\,
   W_{QK,r}\, W_{OV,w}\big) \nonumber \\
&= \operatorname{tr}\big(W_{V,w} W_{O,w}^{\top}\;
   W_{K,r} W_{Q,r}^{\top}\; W_{Q,r} W_{K,r}^{\top}\;
   W_{O,w} W_{V,w}^{\top}\big) \nonumber \\
&= \operatorname{tr}\big(X^{\top} g_{Q}\, X\, g_{V}\big),
\qquad X := W_{K,r}^{\top} W_{O,w}.
\label{eq:factored}
\end{align}
The last step is one cyclic move and a regrouping. Cycling the
leading $W_{V,w}$ to the end joins it with its transpose as the
inner Gram $g_V := W_{V,w}^{\top} W_{V,w}$. The central pair
$W_{Q,r}^{\top} W_{Q,r}$ is $g_Q$, and the two mixed products
flanking $g_Q$ are $X^{\top} = W_{O,w}^{\top} W_{K,r}$ and $X$.
Every matrix in the final expression is
$d_{\mathrm{head}} \times d_{\mathrm{head}}$, and the stream
width $d$ survives only inside the single product that builds
$X$. The denominators collapse by the same move,
$\lVert W_{QK,r} \rVert_F^{2}
= \operatorname{tr}(W_{K,r} W_{Q,r}^{\top} W_{Q,r} W_{K,r}^{\top})
= \operatorname{tr}(g_Q\, g_K)$ and
$\lVert W_{OV,w} \rVert_F^{2} = \operatorname{tr}(g_O\, g_V)$, so
one elementwise division and a square root finish $C$ for every
pair. Table~\ref{tab:channels} instantiates the numerator for all
three channels.

\begin{table}[h]
\centering
\caption{\textbf{The factored numerator for each head-to-head
channel.} The writer always contributes $W_{O,w}$ and $g_{V,w}$.
The channels differ in which reader factor forms the cross factor
$X$ and which reader inner Gram weights it.}
\label{tab:channels}
\footnotesize
\begin{tabular}{llll}
\toprule
Channel & Reader $R$ & Cross factor $X$ & $\lVert R\,W_{OV,w} \rVert_F^2$ \\
\midrule
K-composition & $W_{QK,r}$ & $W_{K,r}^{\top} W_{O,w}$ & $\operatorname{tr}(X^{\top} g_{Q,r}\, X\, g_{V,w})$ \\
Q-composition & $W_{QK,r}^{\top}$ & $W_{Q,r}^{\top} W_{O,w}$ & $\operatorname{tr}(X^{\top} g_{K,r}\, X\, g_{V,w})$ \\
V-composition & $W_{OV,r}$ & $W_{V,r}^{\top} W_{O,w}$ & $\operatorname{tr}(X^{\top} g_{O,r}\, X\, g_{V,w})$ \\
\bottomrule
\end{tabular}
\end{table}

Every ingredient of the score is now a cached
$d_{\mathrm{head}}$-wide object, and the coefficient assembles
as, for a K-composition edge,
\begin{equation}
C(r, w) \;=\;
\sqrt{\frac{\operatorname{tr}\!\big(X^{\top} g_{Q,r}\, X\, g_{V,w}\big)}
{\operatorname{tr}(g_{Q,r}\, g_{K,r})\;
 \operatorname{tr}(g_{O,w}\, g_{V,w})}},
\qquad X = W_{K,r}^{\top} W_{O,w},
\label{eq:assembled}
\end{equation}
with the other channels substituting their row of
Table~\ref{tab:channels}. The four inner-Gram stacks are computed
once per model and reused by every pair, and batched tensor
operations evaluate the formula for all ordered pairs of the
model's $N_h$ heads at once. The released code is the reference
for the exact batching.

Against the naive route's ${\sim}10^{13}$ operations, evaluating
all three GPT-2 channels this way costs about
$2 \times 10^{11}$ multiply-adds, and the gap widens as
$(d/d_{\mathrm{head}})^{2}$ with width, the same comparison at
Pythia-6.9B being ${\sim}2 \times 10^{14}$ against
${\sim}10^{17}$. The memory gap is as decisive as the arithmetic one. The naive route's composed operators are $d^{2}$ floats each, $137$ GB in total at Pythia-6.9B against the card's $32$ GB. The factored route's resident tensors, the factor and Gram stacks and one cross-factor block, total a few gigabytes at the same scale. Measured, the $C$ tables of all three channels
fill in $0.1$ seconds for GPT-2 and $14$ seconds for Pythia-6.9B
on one GPU.

The interface-matrix $\leftrightarrow$ head classes evaluate the
same master formula with the interface matrix's materialized Gram
(Table~\ref{tab:rw}) on one side. The head side's Gram is the
sandwich $G_r = W_{K,r}\, g_{Q,r}\, W_{K,r}^{\top}$
(K-composition; the other channels analogously), so
$\operatorname{tr}(G_r H) = \operatorname{tr}\big(g_{Q,r}\,
(W_{K,r}^{\top} H\, W_{K,r})\big)$ and the head Gram is again
never formed. These classes are always
causally legal because the embedding and positional writers act
before the first layer and the unembedding reader acts after the
last.

\subsection{Mixed pairs, quadratic forms through the skinny factors}
\label{app:mixed}

When exactly one side of the pair is a neuron, as in the head $\leftrightarrow$ neuron channels and the interface-matrix $\leftrightarrow$ neuron pairs, the score collapses from the dense machinery of the previous subsection to a small quadratic form. A neuron reads and writes through single vectors, so one side of every such pair is a single row or column whose Gram has rank one. We derive head $\to$ neuron in full, because the remaining
mixed classes follow the same pattern.

The edge asks how strongly the write operator of head $w$
overlaps the input direction of neuron $m$, so the reader is the
single row $R = \hat{r}_m^{\top}$ and the writer is
$W = W_{OV,w}$. Their product
$RW = \hat{r}_m^{\top} W_{OV,w}$ is a
$[1 \times d][d \times d]$ multiplication, hence a single
row of length $d$, and the Frobenius norm of a single row is
that vector's Euclidean length, so
$\lVert RW \rVert_F = \lVert W_{OV}^{\top} \hat{r}_m \rVert_2$.
Evaluating this norm never requires the $d \times d$ operator
itself. Substituting $W_{OV}^{\top} = W_V W_O^{\top}$ and
multiplying from the right,
\begin{equation}
W_{OV}^{\top} \hat{r}_m = W_V \big(W_O^{\top} \hat{r}_m\big)
= W_V\, x_m,
\qquad x_m := W_O^{\top} \hat{r}_m \in \R^{d_{\mathrm{head}}},
\label{eq:mixedproj}
\end{equation}
the computation passes through the $d_{\mathrm{head}}$-vector $x_m$, which is the
neuron's read direction expressed in the head's output channels.
The squared numerator then collapses onto the cached inner Gram,
$\lVert W_V x_m \rVert_2^2 = x_m^{\top} g_V\, x_m$. Both
denominators are already precomputed as well, since
$\lVert \hat{r}_m \rVert_2 = 1$ by construction and
$\lVert W_{OV} \rVert_F^2 = \operatorname{tr}(g_O g_V)$, so the
edge's coefficient assembles entirely from cached
$d_{\mathrm{head}}$-wide objects,
\begin{equation}
C(w, m) \;=\;
\sqrt{\frac{x_m^{\top} g_{V,w}\, x_m}
{\operatorname{tr}(g_{O,w}\, g_{V,w})}},
\qquad x_m = W_{O,w}^{\top}\, \hat{r}_m,
\label{eq:assembledmixed}
\end{equation}
the mixed-pair analog of Eq.~\ref{eq:assembled}.

\paragraph{Batching.}
To compute the squared numerator $\lVert W_V x_m \rVert_2^2 = x_m^{\top} g_V\, x_m$ between one head and all $N_t = L\, d_{\mathrm{mlp}}$ neurons in a single pass ($36{,}864$ on GPT-2 small), every neuron's unit read vector $\hat{r}_m^{\top} = r_m^{\top} / \lVert r_m \rVert$ is stacked as
the rows of $\hat{R} \in \R^{N_t \times d}$. One matrix
multiplication $X = \hat{R}\, W_O$ then delivers all the
projections at once, row $m$ being $x_m^{\top}$, and batched
tensor operations evaluate the $N_t$ quadratic forms
$x_m^{\top} g_V\, x_m$ from it directly, the released code being
the reference. The denominators need no batching at all. The neuron side of every score is $\lVert R \rVert_F = 1$ by construction, since the rows of $\hat{R}$ are unit vectors. The head side is the single constant
$\lVert W \rVert_F = \lVert W_{OV} \rVert_F
= \sqrt{\operatorname{tr}(g_O g_V)}$ shared by all $N_t$ scores
of the row, so finishing $C$ is one division of the row by one
number. Per head this is about $2 \times 10^{9}$ multiply-adds,
and nothing larger than $N_t \times d_{\mathrm{head}}$ is
ever held.

The head\,$\leftrightarrow$\,neuron classes are computed one head per batch, each pass scoring that head against all $N_t$ neurons as one row of an $[N_h \times N_t]$ table, iterating through the heads until the table is full. The causal mask, applied to the finished table, depends on the block architecture, because the two architectures differ in whether a head can feed the neurons of its own layer. In the sequential-block models (GPT-2 and GPT-Neo), attention writes to the residual stream before the MLP of the same block reads it, so a head can feed the neurons of its own layer, and the mask allows equality, $l_w \le l_r$. By contrast, the Pythia models run attention and the MLP in parallel off the same layer input, so a head cannot feed the neurons of its own layer, and the mask is strict, $l_w < l_r$.

\paragraph{Neuron $\to$ head.}
These channels ask the reverse question, whether a neuron's write
vector lands where a later head reads, and the derivation is the
mirror image of head $\to$ neuron with the writer and reader
roles swapped, so we state the result directly. For
K-composition,
\begin{equation}
C(m, r) \;=\;
\sqrt{\frac{y^{\top} g_{Q,r}\, y}
{\operatorname{tr}(g_{Q,r}\, g_{K,r})}},
\qquad y = W_{K,r}^{\top}\, \hat{w}_m,
\label{eq:assembledmixedrev}
\end{equation}
and the Q- and V-composition channels substitute
$y = W_{Q,r}^{\top} \hat{w}_m$ against $g_{K,r}$ and
$y = W_{V,r}^{\top} \hat{w}_m$ against $g_{O,r}$, with the
normalizers of Table~\ref{tab:rw}. Batching is identical to
head $\to$ neuron with the roles reversed. The mask here is
strict, $l_w < l_r$, because an MLP writes after everything else
in its layer.

\paragraph{Interface matrices $\leftrightarrow$ neurons.}
These classes pair a neuron vector with an interface matrix's Gram, and no factoring is needed at all. These are the three Grams that we materialize (Table~\ref{tab:rw}), $H_{\mathrm{emb}} = W_E
W_E^{\top}$, $H_{\mathrm{pos}} = W_{\mathrm{pos}}
W_{\mathrm{pos}}^{\top}$, and $G_{\mathrm{unemb}} = W_U
W_U^{\top}$, each computed once as a $d \times d$ array when the
weights are loaded. The neuron side is a unit vector, so the
coefficient is a normalized quadratic form in the materialized
Gram,
\begin{equation}
C(H, m) \;=\;
\sqrt{\frac{\hat{r}_m^{\top} H\, \hat{r}_m}{\operatorname{tr} H}},
\qquad
C(m, G_U) \;=\;
\sqrt{\frac{\hat{w}_m^{\top} G_U\, \hat{w}_m}
{\operatorname{tr} G_U}}.
\label{eq:assemblediface}
\end{equation}
The left form serves the embedding and positional writers $H \in \{H_{\mathrm{emb}}, H_{\mathrm{pos}}\}$ against the neuron's read vector. The right form serves the neuron's write vector against the unembedding reader $G_U := G_{\mathrm{unemb}}$. Both are evaluated for all $N_t$ neurons at once by batched tensor operations. The same materialized arrays
also supply their eigenvalues to the exact null distribution below.

\paragraph{Exact rotation null distributions.}
A rank-one reader also makes the rotation null distribution exact
and inexpensive to sample.
Under the null hypothesis the read direction is a uniformly random unit vector. In the eigenbasis of $H$ the statistic $v^{\top} Q H Q^{\top} v = \sum_i \lambda_i u_i^2$ is a weighted sum of the squared coordinates of a uniform unit vector $u$, with $\lambda_i$ the eigenvalues of $H$. Draws from the null distribution therefore sample $u$ directly and never form a $d \times d$ rotation matrix.

\subsection{Rank-one pairs, neuron $\to$ neuron streamed}
\label{app:stream}

For the last and largest class, neuron $\to$ neuron, both the reader and the writer are single vectors, and no general machinery is needed. The product $RW = r^{\top} w$ is a $1 \times 1$ matrix, and the Frobenius norm of a single row or column is the vector's length. The score is therefore a plain cosine directly from the definition,
$C = |r^{\top} w| / (\lVert r \rVert\, \lVert w \rVert)
= |\cos\theta|$,
with $\theta$ the angle between $r$ and $w$. The
master formula agrees, since the middle scalar factors out of the
trace, $\operatorname{tr}(r\, r^{\top} w\, w^{\top}) =
(r^{\top} w)^2$. What makes this class hard is not the formula but
the count, $6.2 \times 10^8$ candidate pairs on GPT-2 small and
$1.3 \times 10^{11}$ at Pythia-6.9B, far too many to hold in
memory at once. The denominators are folded in up front by unit-normalizing each vector once. One matrix product of a layer's $d_{\mathrm{mlp}}$ unit write vectors against a later layer's $d_{\mathrm{mlp}}$ unit read vectors then emits the full $d_{\mathrm{mlp}} \times d_{\mathrm{mlp}}$ tile of cosines for that layer pair. The computation runs one tile at a time through
the $L(L-1)/2$ ordered layer pairs ($66$ tiles of $38$ MB and
$9.4 \times 10^{6}$ cosines each on GPT-2 small). A single pass
suffices, because the class is censused rather than selected
(Section~\ref{sec:map}). Each tile contributes its per-span
histogram for Figure~\ref{fig:nncos}, its pairs beyond the wire
threshold $|\cos| \ge 0.5$, and its maximum. Peak memory is one
$38$ MB tile, and the $6.2 \times 10^8$ scores never exist at
once.

\paragraph{Summary.}
Our efficient implementation of the coupling
coefficient $C$ varies by the type of connection, as follows:
\begin{enumerate}
\item When both sides are full matrices, factored traces through
  the inner Grams score all pairs of the class at once
  (Appendix~\ref{app:gemm}).
\item When one side is a neuron, quadratic forms score one
  component against every neuron at once
  (Appendix~\ref{app:mixed}).
\item When both sides are neurons, the cosines stream through one
  tile at a time.
\end{enumerate}

\subsection{Census and run configuration}

Table~\ref{tab:census} defines the $18$ connection classes with
their causal masks and candidate-count formulas.
Table~\ref{tab:scalecensus} gives the candidate census for every
mapped model. Selection applies only to the head-to-head classes
(Section~\ref{app:selection}). Interface classes are scored as effect sizes against their rotation null distributions without selection, and the neuron classes are censused against
the exact chance law (Section~\ref{sec:map}).

\begin{table}[t]
\centering
\caption{\textbf{Candidate-edge summary of seven mapped models}.
 Counts are causally eligible
ordered pairs, and \textit{wires} are defined as neuron$\to$neuron pairs with
$|\cos| \ge 0.5$. The rotary Pythia models lack the 4 classes where the positional embedding matrix is the writer, and, as parallel-block models, use the strict head$\to$neuron mask (Appendix~\ref{app:mixed}). (ifc = interface matrices).}
\label{tab:scalecensus}
\footnotesize
\setlength{\tabcolsep}{2pt}
\begin{tabular}{lrrrrrrr}
\toprule
Model & head--head & ifc$\leftrightarrow$head &
head$\leftrightarrow$neuron & ifc$\leftrightarrow$neuron &
neuron--neuron & Total & Wires \\
\midrule
GPT-2 & $28{,}512$ & $1{,}008$ & $1.02\times10^{7}$ & $110{,}592$ & $6.23\times10^{8}$ & $6.33\times10^{8}$ & $2{,}066$ \\
GPT-2-medium & $211{,}968$ & $2{,}688$ & $7.39\times10^{7}$ & $294{,}912$ & $4.63\times10^{9}$ & $4.70\times10^{9}$ & $3{,}124$ \\
GPT-2-large & $756{,}000$ & $5{,}040$ & $2.62\times10^{8}$ & $552{,}960$ & $1.65\times10^{10}$ & $1.68\times10^{10}$ & $6{,}001$ \\
GPT-Neo-125M & $28{,}512$ & $1{,}008$ & $1.02\times10^{7}$ & $110{,}592$ & $6.23\times10^{8}$ & $6.33\times10^{8}$ & $376$ \\
Pythia-160m & $28{,}512$ & $576$ & $9.73\times10^{6}$ & $73{,}728$ & $6.23\times10^{8}$ & $6.33\times10^{8}$ & $725$ \\
Pythia-2.8B & $1{,}523{,}712$ & $4{,}096$ & $6.50\times10^{8}$ & $655{,}360$ & $5.20\times10^{10}$ & $5.27\times10^{10}$ & $2{,}389$ \\
Pythia-6.9B & $1{,}523{,}712$ & $4{,}096$ & $1.04\times10^{9}$ & $1{,}048{,}576$ & $1.33\times10^{11}$ & $1.34\times10^{11}$ & $1{,}837$ \\
\bottomrule
\end{tabular}
\end{table}

\begin{table}[t]
\centering
\caption{\textbf{The $18$ connection classes.} Reader--writer
pairing, causal mask, and candidate-count formula, instantiated
for GPT-2 small. The head$_{\{K,Q,V\}}$ rows aggregate three
separately scored channel classes over the same pairs.}
\label{tab:census}
\footnotesize
\setlength{\tabcolsep}{3.5pt}
\begin{tabular}{lll}
\toprule
Connection class & Mask & Count formula \\
\midrule
head $\to$ head$_{\{K,Q,V\}}$ & $l_w < l_r$ & $3 \times 12^2 \times 66$ \\
embedding $\to$ head$_{\{K,Q,V\}}$ & always & $3 \times 144$ \\
positional $\to$ head$_{\{K,Q,V\}}$ & always & $3 \times 144$ \\
head $\to$ unembedding & always & $144$ \\
head $\to$ neuron & $l_w \le l_r$ (sequential); $l_w < l_r$ (parallel, Pythia) & $(12 \times 3072) \times 78$ \\
neuron $\to$ head$_{\{K,Q,V\}}$ & $l_w < l_r$ & $3 \times (3072 \times 12) \times 66$ \\
embedding $\to$ neuron & always & $36{,}864$ \\
positional $\to$ neuron & always & $36{,}864$ \\
neuron $\to$ unembedding & always & $36{,}864$ \\
neuron $\to$ neuron & $l_w < l_r$ & $3072^2 \times 66$ \\
\bottomrule
\end{tabular}
\end{table}

Weights are loaded with LayerNorm folded and writer matrices centered (Appendix~\ref{app:lnfold}).
All decompositions are float64. Natural text is a fixed public sample of pile-10k, $32$ sequences of $256$ tokens, for both the ablation experiments and the subspace deletions. Repeated-block sequences use $128$-token uniform-random blocks, and the random-token sequences behind the estimated positional matrix are $128$ tokens long. All seeds
are fixed. Every experiment in this paper runs on one consumer GPU. The
full GPT-2 map builds in $15$\,s, the cross-model ablation
experiments run in minutes per model, and the full maps of
Pythia-2.8B and 6.9B build in $4$ and $11$ minutes.

\subsection{The census at initialization and across training}
\label{sec:training}\label{app:training}
This appendix gives the census at initialization and across training summarized in Section~\ref{sec:census}. The rotation null distribution is the generating distribution of a randomly initialized model, so a census at initialization is the direct calibration check of the baseline. Table~\ref{tab:censustraining} reports this census for GPT-2 and for Pythia, together with the census along the public training trajectory of Pythia-160m and Pythia-2.8B \citep{biderman2023pythia}. At initialization every channel of every model sits at chance, with $2$--$3\%$ of pairs on either side of the $|z| = 2$ line and group medians near $\pm 2$, which are the tails the null distribution predicts. Structure then develops in two stages. Super-coupling appears within the first thousand steps of Pythia-160m, where $61\%$ of K-composition pairs and $73\%$ of Q-composition pairs are already above chance while none is below. Avoidance develops later, between steps $1{,}000$ and $4{,}000$, and keeps growing to the end of training. The avoidant share rises from $0\%$ to $27\%$ of K-composition pairs and from $7\%$ to $63\%$ of V-composition pairs. Pythia-2.8B shows the same order at a larger scale. At step $1{,}000$, $79$--$84\%$ of K- and Q-composition pairs are already above chance and none is below. The avoidant share of V-composition pairs then grows from $0\%$ at step $1{,}000$ to $9\%$ at step $4{,}000$ and $50\%$ in the released model. Super-coupling and avoidance therefore do not arise together. Avoidance is learned, and learned later. Figure~\ref{fig:censustraining} plots the trajectories.

\begin{table}[h]
\centering
\caption{\textbf{Head-to-head pairs against the theoretical null distribution at initialization and across training.} Same format as Table~\ref{tab:theorycensus}: share of head pairs above ($z \ge 2$) and below ($z \le -2$) chance per composition channel, with the median $z$ of each group in parentheses. GPT-2 (init) is a randomly initialized model of GPT-2's shape. Pythia rows are the public training checkpoints, identified by training step, and the last row of each series is the released model. For Pythia-2.8B, the step-$0$ and step-$16{,}000$ branches of the public repository carry the released weights in every file (their census matches the released model to all digits), so only branches whose weights differ from the released model are shown.}
\label{tab:censustraining}
\footnotesize
\begin{tabular}{lccc}
\toprule
& K-composition & Q-composition & V-composition \\
Model (step) & above / below & above / below & above / below \\
\midrule
GPT-2 (init) & $2\%\,(+2)$ / $2\%\,(-2)$ & $2\%\,(+2)$ / $2\%\,(-2)$ & $2\%\,(+2)$ / $2\%\,(-2)$ \\
GPT-2 (trained) & $55\%\,(+16)$ / $34\%\,(-8)$ & $61\%\,(+18)$ / $27\%\,(-7)$ & $25\%\,(+15)$ / $65\%\,(-10)$ \\
\midrule
Pythia-160m (0) & $3\%\,(+2)$ / $2\%\,(-2)$ & $3\%\,(+2)$ / $2\%\,(-2)$ & $3\%\,(+2)$ / $2\%\,(-2)$ \\
Pythia-160m (1,000) & $61\%\,(+4)$ / $0\%\,(-3)$ & $73\%\,(+5)$ / $0\%\,(-3)$ & $15\%\,(+3)$ / $7\%\,(-3)$ \\
Pythia-160m (4,000) & $48\%\,(+6)$ / $17\%\,(-4)$ & $66\%\,(+8)$ / $8\%\,(-4)$ & $22\%\,(+6)$ / $45\%\,(-6)$ \\
Pythia-160m (16,000) & $58\%\,(+7)$ / $22\%\,(-7)$ & $71\%\,(+9)$ / $15\%\,(-7)$ & $25\%\,(+9)$ / $56\%\,(-9)$ \\
Pythia-160m (64,000) & $52\%\,(+7)$ / $26\%\,(-7)$ & $68\%\,(+10)$ / $15\%\,(-6)$ & $25\%\,(+9)$ / $60\%\,(-10)$ \\
Pythia-160m (143,000) & $50\%\,(+7)$ / $27\%\,(-7)$ & $65\%\,(+9)$ / $16\%\,(-5)$ & $21\%\,(+9)$ / $63\%\,(-10)$ \\
\midrule
Pythia-2.8B (1,000) & $79\%\,(+5)$ / $0\%\,(-2)$ & $84\%\,(+8)$ / $0\%\,(-2)$ & $56\%\,(+4)$ / $0\%\,(-2)$ \\
Pythia-2.8B (4,000) & $80\%\,(+8)$ / $2\%\,(-3)$ & $85\%\,(+10)$ / $1\%\,(-3)$ & $46\%\,(+5)$ / $9\%\,(-3)$ \\
Pythia-2.8B (128,000) & $56\%\,(+9)$ / $13\%\,(-4)$ & $52\%\,(+8)$ / $12\%\,(-3)$ & $27\%\,(+9)$ / $50\%\,(-7)$ \\
Pythia-2.8B (143,000) & $55\%\,(+9)$ / $14\%\,(-4)$ & $52\%\,(+8)$ / $12\%\,(-3)$ & $27\%\,(+9)$ / $50\%\,(-7)$ \\
\bottomrule
\end{tabular}
\end{table}

\begin{figure}[t]
\centering
\includegraphics[width=0.95\linewidth]{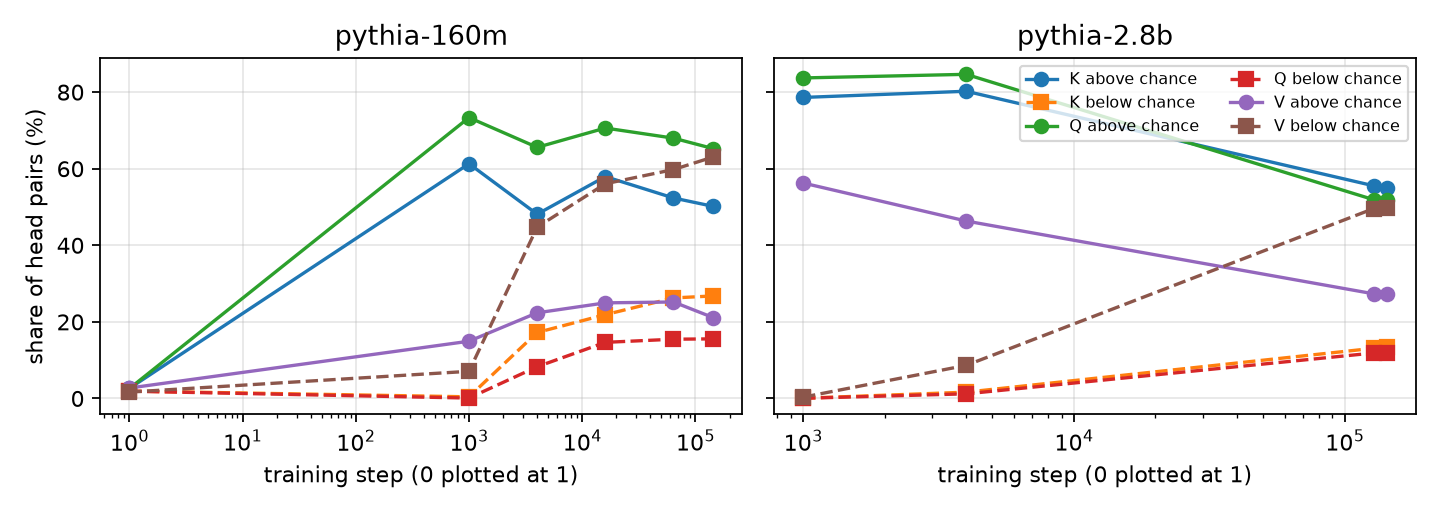}
\caption{\textbf{The census across training.} Share of head pairs above (solid) and below (dashed) chance per composition channel against training step for Pythia-160m (left) and Pythia-2.8B (right). Step $0$ is plotted at $1$ on the logarithmic step axis. Super-coupling appears within the first thousand steps, and avoidance develops later.}
\label{fig:censustraining}
\end{figure}

\section{The head graph and its ablations in full}
\label{app:extra}

\subsection{Selection details}
\label{app:selectiondetails}

\paragraph{Strata.}
A \emph{stratum} is the set of candidate edges that are calibrated
together, and every candidate edge belongs to exactly one stratum. Two coordinates define a stratum. One coordinate is the connection class, because different kinds of pairs have different chance behavior. A cosine
between two rank-one neuron vectors, a head circuit read against a
neuron vector, and two head circuits are geometrically different
statistics. The other coordinate is the layer separation between writer and reader. The stream's geometry drifts with depth, so nearby components share more ambient structure than distant ones.
``K-composition edges whose writer sits three layers below the
reader'' is one stratum, and ``neuron$\to$neuron pairs seven layers
apart'' is another. On GPT-2 each head--head channel contributes eleven strata, one per separation. Each stratum receives its own empirical null distribution, with its own median and robust standard deviation, so an edge is always judged against chance for its own kind of pair at its own distance. Pooling the separations would instead judge every pair against a null distribution dominated by the most numerous separations.

\paragraph{The empirical null distribution.}
The stratum's empirical spread is the robust estimate of
Eq.~\ref{eq:empnull},
$s_{\mathrm{emp}} = 1.4826 \cdot \operatorname{MAD}_s$. For the
neuron--neuron class, where Lemma~\ref{lem:beta} supplies the
theoretical spread $s_{\mathrm{theory}} = 1/\sqrt{d}$ in closed
form, the empirical spread agrees within $3\%$ (script released).
For rank-one pairs the empirical and theoretical null distributions therefore coincide, and for subspace pairs they diverge. Selection thus standardizes against the empirical null distribution, fitted per stratum (edge class $\times$ layer span) by central matching \citep{efron2004large}. In central matching, the central quantiles of the stratum's score distribution estimate the center and the standard deviation of the null distribution, on the assumption that true channels are sparse in the bulk. On GPT-2 the resulting head-class thresholds fall between $z = 2.8$ and $z = 3.2$, the $99.75$th to $99.94$th percentiles of the fitted empirical null distributions. The released code estimates the spread in two robust ways. The GPT-2 head graph of Section~\ref{sec:headgraph} uses the MAD form of Eq.~\ref{eq:empnull}, whereas the cross-model community pipeline behind Table~\ref{tab:communityrep} and the neuron streaming code use the interquartile range divided by $1.349$, the corresponding consistency factor for the Gaussian. On GPT-2 the two spread estimates agree to a median ratio of $1.04$ across the $33$ head strata, and the two selections nearly coincide. The interquartile form selects $987$ head edges, $986$ of which are among the $1{,}051$ edges of the MAD form.

\paragraph{Why robust estimators.}
The robust estimators of Eq.~\ref{eq:empnull} are not a convention. They are a measured necessity. Refitting the same strata with the stratum mean and standard deviation lets the signal tail inflate the estimated standard deviation of the null distribution by up to $5\times$. That refit removes $61\%$ of the head--head map's edges, with the signal corrupting its own null distribution (script released). At scale, this selection standard keeps $115{,}164$ of
the $1{,}523{,}712$ head--head candidates at $2.8$B and
$115{,}793$ at $6.9$B.

\subsection{Community detection}
\label{app:communities}
The head graph of Figure~\ref{fig:mapoverview} is an undirected graph over GPT-2's $144$ heads. Two heads are linked when at least one of the three channels between them is selected, and the link is weighted by the number of selected channels, from $1$ to $3$. Of the $144$ heads, $133$ fall in the giant connected component, which has $717$ links, $326$ of them carrying more than one selected channel. We partition the giant component with the Louvain algorithm \citep{blondel2008fast}, as implemented in NetworkX, at the default resolution of $1$ and with seed $0$, and Figure~\ref{fig:mapoverview} shows that partition. It has five communities of $38$, $30$, $25$, $21$, and $19$ heads and a modularity of $0.283$. Modularity is the weighted quantity of \citet{newman2004finding},
\begin{equation}
Q \;=\; \frac{1}{2W} \sum_{i,j} \Big( w_{ij} - \frac{k_i\, k_j}{2W} \Big)\, \delta(c_i, c_j),
\label{eq:modularity}
\end{equation}
where $w_{ij}$ is the link weight, $k_i = \sum_j w_{ij}$, $W = \tfrac12 \sum_{i,j} w_{ij}$, and $\delta(c_i, c_j)$ is one when heads $i$ and $j$ share a community and zero otherwise. Across ten Louvain seeds, nine partitions have five communities and one has six, the modularity ranges from $0.264$ to $0.285$, and the adjusted Rand index between the seed-$0$ partition and each of the other nine averages $0.84$ with a minimum of $0.67$. All five induction heads fall in a single community in every seed. The IOI name-mover heads L9H6 and L9H9 join them in every seed, and L10H0 never does.

\subsection{Layer-separation profile}
\label{app:span}

Table~\ref{tab:span} gives the linked-pair counts by writer-to-reader layer separation behind the long-range pattern of Section~\ref{sec:headgraph}.

\begin{table}[h]
\centering
\caption{Linked head pairs by writer-to-reader layer separation.
Eligible pairs are all ordered pairs with the writer in a strictly
earlier layer. A pair is linked if at least one of its three channels
is selected under $C$ (FDR $q = 0.05$).}
\label{tab:span}
\small
\begin{tabular}{lrrrrrrrrrrr}
\toprule
layer separation & 1 & 2 & 3 & 4 & 5 & 6 & 7 & 8 & 9 & 10 & 11 \\
\midrule
eligible pairs & 1584 & 1440 & 1296 & 1152 & 1008 & 864 & 720 & 576 & 432 & 288 & 144 \\
linked pairs & 102 & 88 & 77 & 65 & 59 & 69 & 83 & 68 & 48 & 40 & 18 \\
linked (\%) & 6.4 & 6.1 & 5.9 & 5.6 & 5.9 & 8.0 & 11.5 & 11.8 & 11.1 & 13.9 & 12.5 \\
\bottomrule
\end{tabular}
\end{table}

\subsection{The community ablation in full}
\label{app:ablationfull}

\paragraph{Protocol details.}
The ablated community is the archived seed-$0$ partition drawn in Figure~\ref{fig:mapoverview}. Where a layer holds more community heads than other heads, exact layer matching forces a control to include community heads from that layer. Table~\ref{tab:communityrep} reports that overlap and adds a size-matched control drawn from non-community heads only. The readout is the induction gain of Eq.~\ref{eq:gain}, and every table reports the destroyed share defined in Section~\ref{app:communityrep}.

\begin{table}[!ht]
\centering
\caption{\textbf{The community ablation replicates across models.}
Percent of the clean induction gain destroyed by mean-ablating each
head set. Community is the induction-concentrated community of size $n$, and ``holds'' is how many of the model's top-five behavioral
induction heads fall in the community. Five heads is the ablation of those five behavioral induction heads alone, Comm.\ minus five is the ablation of the community with those heads removed, and Whole comm.\ is the ablation of the full community. Ctrl (layer) is the median over
five random sets matched to the whole community in size and per-layer
head counts. Where a layer holds more community heads than other heads,
the shortfall is filled from the community's own heads in that layer
($\dagger$: $10$ of $184$ heads at GPT-2-large, $2$ of $32$ at GPT-Neo,
$222$ of $443$ at Pythia-2.8B). Ctrl (size) is the median over five
random sets of the same size drawn from non-community heads with no
layer constraint. Communities come from each model's own selected head
graph (protocol below). Pythia-6.9B is absent due to our hardware's RAM
constraint.}
\label{tab:communityrep}
\footnotesize
\setlength{\tabcolsep}{2pt}
\begin{tabular}{lcccccc}
\toprule
Model & Community (holds) & Five heads & Comm.\ minus five & Whole comm. & Ctrl (layer) & Ctrl (size) \\
\midrule
GPT-2 & $n{=}30$ (5/5) & $60.8\%$ & $72.9\%$ & $94.7\%$ & $2.9\%$ & $2.4\%$ \\
GPT-2-medium & $n{=}47$ (5/5) & $4.7\%$ & $75.4\%$ & $93.4\%$ & $2.0\%$ & $2.0\%$ \\
GPT-2-large & $n{=}184$ (4/5) & $3.6\%$ & $102.0\%$ & $108.5\%$ & $20.3\%^{\dagger}$ & $23.9\%$ \\
GPT-Neo-125M & $n{=}32$ (5/5) & $63.7\%$ & $100.5\%$ & $101.3\%$ & $97.1\%^{\dagger}$ & $63.5\%$ \\
Pythia-160m & $n{=}25$ (3/5) & $60.0\%$ & $92.1\%$ & $86.2\%$ & $27.4\%$ & $6.9\%$ \\
Pythia-2.8B & $n{=}443$ (4/5) & $41.3\%$ & $96.9\%$ & $97.1\%$ & $90.5\%^{\dagger}$ & $2.9\%$ \\
\bottomrule
\end{tabular}
\end{table}

\paragraph{The nested decomposition.}
The GPT-2 decomposition ablates five nested head sets (Table~\ref{tab:decomp}). They run from the five induction heads alone, through the previous-token head L4H11 and the $24$ remaining community heads, up to the full community. If induction were carried solely by the
five induction heads, removing them would abolish the capability
and removing the rest of the community would have little effect.
Table~\ref{tab:decomp} rejects both halves of that null hypothesis.

\paragraph{Non-additivity.}
The whole community destroys far more than its parts sum to (Table~\ref{tab:decomp}). Such non-additivity is the signature of redundancy, documented as backup heads, self-repair, and non-unique task circuits
\citep{wang2022interpretability,mcgrath2023hydra,gong2026conditional,chen2026circuits}.
The decomposition places that redundant machinery, invisible to
single-head identification, inside a boundary the map drew from
weights alone.

\begin{table}[h]
\centering
\caption{\textbf{Decomposing the induction-community knockout in
GPT-2.} Fraction of the clean induction gain ($12.40$ nats) destroyed
by mean-ablating each head set. Each ablated set is compared with its
own five random control sets, matched to it in size and per-layer
head counts. Pile $\Delta$NLL is the natural-text
loss increase in nats. All numbers are stable across corpus seeds
(largest change $1.1$ points).}
\label{tab:decomp}
\footnotesize
\begin{tabular}{lrrrrr}
\toprule
Ablated set & $n$ & Gain destroyed & Ctrl median & Ctrl max & Pile $\Delta$NLL \\
\midrule
five induction heads & 5 & $60.8\%$ & $0.5\%$ & $1.4\%$ & $+0.09$ \\
previous-token head L4H11 & 1 & $1.9\%$ & $0.1\%$ & $0.6\%$ & $+0.03$ \\
remaining community heads & 24 & $13.5\%$ & $5.6\%$ & $36.1\%$ & $+0.21$ \\
community minus induction heads & 25 & $72.9\%$ & $2.8\%$ & $3.5\%$ & $+0.31$ \\
whole community & 30 & $94.7\%$ & $1.9\%$ & $4.9\%$ & $+0.51$ \\
\bottomrule
\end{tabular}
\end{table}

\paragraph{Ablation mechanics.}
We remove heads by \emph{mean-ablation} of the value summary, not by zeroing. For head $h$, let
$z_h(t) \in \mathbb{R}^{d_{\mathrm{head}}}$ denote its
attention-weighted value vector at position $t$, the quantity the
output factor $W_O$ turns into the head's write
(Appendix~\ref{app:conventions}). A clean forward pass over a
reference corpus of $N$ prompts of $T$ positions defines the frozen
summary
$\bar{z}_h = \frac{1}{NT} \sum_{n=1}^{N} \sum_{t=1}^{T}
z_h^{(n)}(t)$. The ablation replaces $z_h(t) \to \bar{z}_h$ at every position, so the head still writes the constant $W_O \bar{z}_h$ into the
stream but carries no input-dependent signal. The reference corpus
is matched to the readout, the repeated-block corpus for the
induction gain and the Pile sample for the natural-text loss, so
the preserved average is the right one for the distribution being
measured. Mean-ablation and zeroing both remove the head's input-dependent signal, but they differ in what they do to the head's average output. On the one hand, mean-ablation keeps the average output in place. That average acts as a standing bias that downstream LayerNorms and readers are calibrated to. On the other hand, zeroing removes that bias along with the signal. The result is off-distribution damage that inflates apparent importance \citep{wang2022interpretability,zhang2024towards}. Mean-ablation is therefore the conservative choice. The mean is
taken over positions as well as prompts, so any positional
structure in the head's output is removed rather than preserved,
the stricter choice for a position-critical capability. Control sets are ablated identically.

\paragraph{Cross-model protocol.}
We repeat the community ablation on six models spanning three families and both positional schemes
(\citealp{radford2019language,biderman2023pythia,black2021gptneo};
Table~\ref{tab:communityrep}). Every row uses the unified
pipeline. Communities are the Louvain partition (seed $0$) of the model's own selected head graph. The top five induction heads are identified behaviorally by their induction-offset attention on the repeated random blocks of Eq.~\ref{eq:gain}. For head $h$ with attention pattern $A_h$, the score is
\begin{equation}
\mathrm{ind}(h) \;=\; \operatorname*{mean}_{i \in B_2}\; A_h\big(i,\; i - (T - 1)\big),
\label{eq:indscore}
\end{equation}
the attention that a second-block position $i$ pays to the position following the earlier occurrence of its current token, averaged over the positions $i \ge 2(T-1)$ and over the $32$ prompts. The five heads with the highest score are taken (Table~\ref{tab:indheads}). We mean-ablate three head sets per model, namely the five behavioral heads, the community minus those heads, and the whole community, each against matched random controls as above.
The GPT-2 row uses the same archived partition as
Table~\ref{tab:decomp} and Figure~\ref{fig:mapoverview}, so its ablation
numbers coincide with Table~\ref{tab:decomp}, and only the control draws
differ. Pythia-6.9B is absent for a
hardware reason. Assembling the runnable fp32 model requires the
HuggingFace copy and the TransformerLens processing copy simultaneously, about $55$~GB
against our machine's $45$~GB of memory. Map building avoids this
by streaming weights, but ablation requires the assembled model,
so the only intervention on the $6.9$B map is the application's subspace deletion (Section~\ref{sec:interventions}).

\begin{table}[h]
\centering
\caption{\textbf{The five behavioral induction heads of each model}, ranked by induction-offset attention on repeated random blocks. Table~\ref{tab:communityrep} reports how many of them fall inside the induction-concentrated community.}
\label{tab:indheads}
\footnotesize
\begin{tabular}{ll}
\toprule
Model & Induction heads \\
\midrule
GPT-2 & L5H1, L5H5, L6H9, L7H2, L7H10 \\
GPT-2-medium & L7H2, L9H9, L11H1, L12H1, L18H5 \\
GPT-2-large & L16H0, L16H9, L17H4, L19H4, L22H17 \\
GPT-Neo-125M & L6H1, L6H6, L8H0, L8H1, L8H8 \\
Pythia-160m & L4H6, L4H10, L5H0, L5H6, L8H2 \\
Pythia-2.8B & L3H11, L5H1, L5H5, L6H2, L13H24 \\
\bottomrule
\end{tabular}
\end{table}

\paragraph{Cross-model observations.}
Appendix~\ref{app:app1losses} reports the two block losses behind each gain and the ablation restricted to the induction heads inside each community. Two regularities stand out. First, ablating the five behaviorally identified induction heads never abolishes the capability, destroying between $3.6\%$ and $63.7\%$ of the gain. The larger models are the extreme cases of compensation ($4.7\%$ at GPT-2-medium, $3.6\%$ at GPT-2-large). Second, in the four models where the five heads matter on their own, the parts overlap heavily, since the five heads and the community minus those five sum to well over the whole community's effect. Three model-specific
observations accompany the table. At Pythia-2.8B the community holds
$443$ of $1{,}024$ heads and more than half of the heads in $17$ layers,
so no disjoint layer-matched control exists. The layer-matched control there overlaps the community in $222$ heads and destroys $90.5\%$. The size-matched control drawn from the remaining heads destroys $2.9\%$, against the community's $97.1\%$. At GPT-2-large the seed-$0$ community is a quarter of the model's heads ($n = 184$ of $720$). Its controls are therefore themselves elevated ($20.3\%$ and $23.9\%$), and its ablation overshoots ($108.5\%$, at $+4.8$ nats of natural-text cost).
In GPT-Neo the random controls reach $97.1\%$ layer-matched and $63.5\%$ size-matched. This is consistent with GPT-Neo's fragile stream geometry, in which $43\%$ of residual variance lies in a single dimension
\citep[cf.][]{timkey2021all}. We therefore report GPT-Neo as a caveat rather than as evidence.

\subsection{Neuron writers in full}
\label{app:neuronwriters}

Which neurons write into the induction heads, and does ablating them matter? For each of the five behavioral induction heads of GPT-2, we read off the K-, Q-, and V-composition coupling of every causally eligible neuron writer from the map's neuron$\to$head classes (Appendix~\ref{app:mixed}), and we rank neurons by their combined coupling into the five heads,
\begin{equation}
s_m \;=\; \sum_{h \in \mathcal{I}} \;\sum_{c \in \{K, Q, V\}} C^{2}_{c}(m \to h),
\label{eq:neuronscore}
\end{equation}
where $\mathcal{I}$ is the set of five induction heads, $C^{2}_{c}(m \to h)$ is the squared coupling of neuron $m$ into head $h$ through channel $c$, and an ineligible pair contributes zero. Two things stand out. First, the strongest neuron writers concentrate in layer $4$, the layer of the previous-token head. Of the $200$ top writers into each induction head, between $76$ and $142$ sit in layer $4$. Second, induction heads are among the least neuron-fed heads of their layers. Ranked by the mean coupling of their top $200$ neuron writers, four of the five induction heads fall in the bottom three of the twelve heads in their layer (ranks $10$, $6$, $12$, $12$, and $11$). Induction heads listen to heads rather than to neurons. This is the premise behind the head-level focus of the validation and the application.

Ablating the top neuron writers tests whether the ranking carries real signal. We mean-ablate the top $N$ neurons by combined coupling into the induction heads, each against five random neuron sets matched in size and per-layer counts, with the same corpora and readouts as the head ablations (Table~\ref{tab:neuronwriters}). No small set of neurons is critical, since the top $100$ destroy under $1\%$ of the induction gain on either model. The map's ranking nevertheless finds a specific population. The top $1{,}000$ neurons, $2.7\%$ of all neurons, destroy $19.5\%$ of GPT-2's induction gain against $0.1\%$ for the matched random sets (maximum $0.9\%$), at a natural-text cost of $+0.11$ nats. On Pythia-160m the same selection destroys $37.1\%$ against $5.1\%$ (maximum $9.2\%$), though at a larger natural-text cost of $+1.42$ nats. The neuron support of induction is therefore real but diffuse. It is spread over hundreds of neurons rather than carried by a few, and the map locates it from weights alone. This is a conservative test of the neuron classes, because induction is a head-dominated capability. Induction tests the neuron classes on the heads' ground, where the heads carry most of the capability. By contrast, capabilities carried by the MLPs, such as factual recall \citep{geva2021transformer,meng2022locating}, would test the neuron classes on their own ground. We leave that test to future work.

\begin{table}[h]
\centering
\caption{\textbf{Ablating the top neuron writers into the induction heads.} Neurons are ranked by their combined K-, Q-, and V-composition coupling into the five behavioral induction heads, read from the map (Eq.~\ref{eq:neuronscore}). Each row mean-ablates the top $N$ neurons and reports the share of the clean induction gain destroyed and the Pile loss increase ($\Delta$NLL, nats), against the median and maximum over five random neuron sets matched in size and per-layer counts. $N = 1{,}000$ is $2.7\%$ of the neurons of either model.}
\label{tab:neuronwriters}
\footnotesize
\begin{tabular}{lrrrrrr}
\toprule
Model & $N$ & Gain destroyed & Ctrl median & Ctrl max & $\Delta$NLL & Ctrl $\Delta$NLL \\
\midrule
GPT-2 & $100$ & $0.6\%$ & $0.0\%$ & $0.2\%$ & $+0.01$ & $+0.01$ \\
GPT-2 & $300$ & $3.4\%$ & $0.1\%$ & $2.0\%$ & $+0.00$ & $+0.01$ \\
GPT-2 & $1{,}000$ & $19.5\%$ & $0.1\%$ & $0.9\%$ & $+0.11$ & $+0.05$ \\
Pythia-160m & $100$ & $0.2\%$ & $0.3\%$ & $1.5\%$ & $+0.04$ & $+0.01$ \\
Pythia-160m & $300$ & $4.2\%$ & $0.9\%$ & $17.9\%$ & $+0.13$ & $+0.02$ \\
Pythia-160m & $1{,}000$ & $37.1\%$ & $5.1\%$ & $9.2\%$ & $+1.42$ & $+0.09$ \\
\bottomrule
\end{tabular}
\end{table}

\subsection{Block losses and the in-community induction heads}
\label{app:app1losses}

Table~\ref{tab:app1losses} separates the two losses behind the induction gains of Table~\ref{tab:communityrep}. Ablating the whole community leaves the first-block loss within $0.2$ nats of clean on GPT-2 and GPT-2-medium, and it raises the second-block loss to the first-block level in every model. The gain therefore collapses because copying stops, not because the first block becomes easier. The last two columns restrict the five-head ablation to the behavioral induction heads that lie inside the community. Where some lie outside (GPT-2-large, Pythia-160m, Pythia-2.8B), the in-community heads alone destroy less than the full five, and the community-minus-five rows of Table~\ref{tab:communityrep} already exclude exactly these heads, so the compensation reading of Section~\ref{app:communityrep} stands.

\begin{table}[h]
\centering
\caption{\textbf{First-block / second-block losses (nats) behind Table~\ref{tab:communityrep}, and the in-community induction heads.} ``In'' is the number of the five behavioral induction heads that lie inside the community, and the last column is the induction gain destroyed by ablating only those heads.}
\label{tab:app1losses}
\footnotesize
\setlength{\tabcolsep}{3pt}
\begin{tabular}{lcccccc}
\toprule
Model & Clean & Five heads & Comm.\ minus five & Whole comm. & In & In-community heads \\
\midrule
GPT-2 & $12.61$ / $0.21$ & $12.51$ / $7.65$ & $12.42$ / $9.06$ & $12.41$ / $11.76$ & 5 & $60.8\%$ \\
GPT-2-medium & $12.60$ / $0.20$ & $12.44$ / $0.62$ & $12.54$ / $9.48$ & $12.52$ / $11.70$ & 5 & $4.7\%$ \\
GPT-2-large & $12.56$ / $0.24$ & $12.46$ / $0.59$ & $14.84$ / $15.09$ & $14.73$ / $15.77$ & 4 & $1.0\%$ \\
GPT-Neo-125M & $13.34$ / $0.42$ & $13.81$ / $9.11$ & $14.61$ / $14.67$ & $15.16$ / $15.33$ & 5 & $63.7\%$ \\
Pythia-160m & $24.55$ / $7.33$ & $24.22$ / $17.34$ & $23.95$ / $22.58$ & $24.46$ / $22.10$ & 3 & $18.5\%$ \\
Pythia-2.8B & $12.78$ / $0.41$ & $12.69$ / $5.43$ & $12.75$ / $12.36$ & $12.68$ / $12.31$ & 4 & $35.7\%$ \\
\bottomrule
\end{tabular}
\end{table}

\section{The subspace deletion in full}
\label{app:universality}

\begin{figure}[t]
\centering
\includegraphics[width=0.55\linewidth]{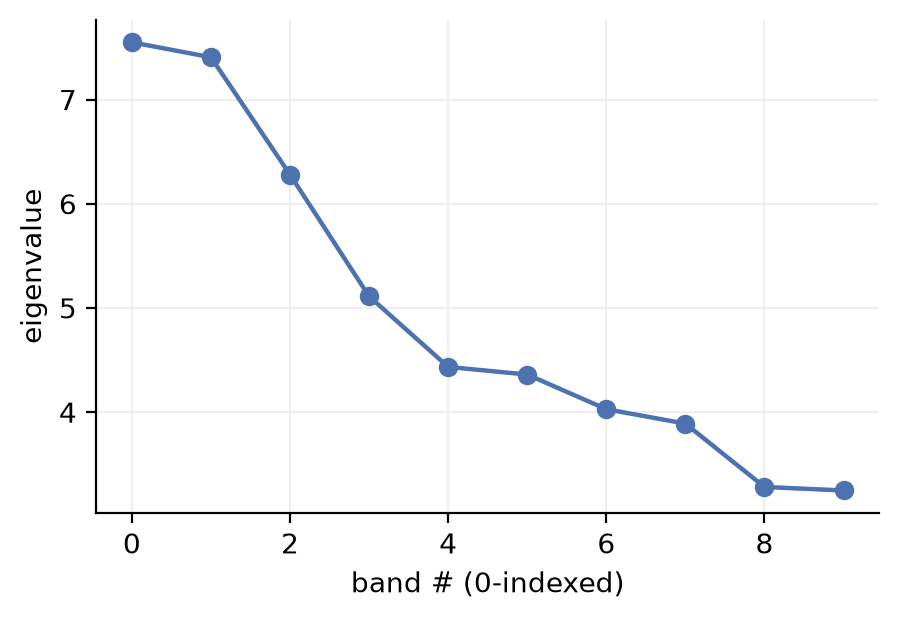}
\caption{\textbf{Eigenvalues of the pooled coupling matrix}
(GPT-2 small). Smooth decay with no sharp cutoff. The ten leading
bands are the candidate pool for the selection rule of Section~\ref{app:bandselect}.}
\label{fig:bands}
\end{figure}

\paragraph{The estimated positional matrix.}
For the rotary models the positional matrix is estimated from activations. The model is fed $32$ sequences of $128$ tokens, each a BOS token followed by $127$ uniform-random tokens, so token content is pure noise and position is the only systematic variable. The residual-stream state $h^{(s)}_t \in \mathbb{R}^{d}$ at the middle layer's input is cached for every sequence $s$ and position $t$. The resulting estimated positional matrix $\hat{\Pi} \in \mathbb{R}^{d \times T}$ has columns
\begin{equation}
\hat{\Pi}_t \;=\; \frac{1}{32} \sum_{s=1}^{32} h^{(s)}_t \;-\; \bar{h},
\qquad
\bar{h} \;:=\; \frac{1}{32\,T} \sum_{s=1}^{32} \sum_{t=1}^{T}
h^{(s)}_t,
\label{eq:posprofile}
\end{equation}
with $T = 128$. Averaging over sequences cancels token content and
keeps exactly what depends on position, and subtracting the mean $\bar{h}$
removes the position-independent component every position shares.
The positional coupling of a band (Section~\ref{app:bandselect}) then uses the estimate,
\begin{equation}
C^{2}_{\mathrm{pos}}(v_i) \;=\;
\frac{\lVert \hat{\Pi}^{\top} v_i \rVert^{2}}
{\lVert \hat{\Pi} \rVert_F^{2}}.
\label{eq:sharerot}
\end{equation}

\paragraph{Robustness checks in full.}
Section~\ref{app:rulerobustness} summarizes three checks on the band selection rule and the estimated positional matrix. First, on Pythia-160m we select the top two directions by eigenvalue alone, ignoring whether a direction carries more positional or token information. Deleting these two leading eigendirections destroys only $46\%$ of the induction gain, against the $95.8\%$ that the $\operatorname{PosRatio}$ selection achieves on the same model. This is consistent with induction being largely a positional phenomenon. It also confirms the importance of $\operatorname{PosRatio}$ as a selection criterion.

Second, we vary the sample size behind the estimated positional matrix $\hat{\Pi}$, building it from $4$ to $128$ random-token sequences at three seeds per size on Pythia-160m. Every run selects the identical pair of top induction-critical directions, with the selected bands' positional couplings never below $125\times$ chance and no unselected band above $28\times$. The estimation protocol for $\hat{\Pi}$ is therefore insensitive to sample size.

Third, we exclude GPT-Neo-125M from Table~\ref{tab:cs2rep} because random two-dimensional deletions already destroy $60$--$98\%$ of its induction gain. No two-dimensional deletion can be distinguished from a random one on this model, so the subspace deletion experiment is not meaningful for it.

The induction gain is a difference of two losses, so it can shrink because the second-block loss rises, because the first-block loss falls, or both. Table~\ref{tab:blocklosses} separates the two. Under the RW-PCA deletion the first-block loss moves by at most $1.1$ nats on the GPT-2 family and by $2.8$--$4.0$ nats on the Pythia models. The second-block loss rises in every model to within $0.8$--$2.2$ nats of the first-block loss. The model no longer profits from the repeat, which is what abolished copying looks like. Activation PCA raises the first-block loss on the GPT-2 family by $3.1$--$7.3$ nats, so part of its gain reduction is general damage, consistent with its larger Pile loss increase in Table~\ref{tab:cs2rep}. The random deletions leave both losses unchanged.

Table~\ref{tab:poscos} completes the principal-cosine record of Table~\ref{tab:cs2rep} with the cosines between the RW-PCA plane and the positional-PCA plane. Table~\ref{tab:dose} reports the deletion-dose check on the three rotary models, where the deletion grows from the top two bands by $\operatorname{PosRatio}$ to the top three and four. On Pythia-160m and Pythia-6.9B the third and fourth bands add nothing to the destroyed gain, and on Pythia-2.8B they add $12$ to $14$ points, whereas the natural-text cost rises by $0.5$ to $3.8$ nats in every model. Two dimensions is therefore the minimal effective choice at these scales.

\begin{table}[h]
\centering
\caption{\textbf{Principal cosines between the RW-PCA plane and the positional-PCA plane.} Same convention as the last two columns of Table~\ref{tab:cs2rep} ($1$ = identical, $0$ = orthogonal).}
\label{tab:poscos}
\footnotesize
\begin{tabular}{lcccccc}
\toprule
& GPT-2 & GPT-2-medium & GPT-2-large & Pythia-160m & Pythia-2.8B & Pythia-6.9B \\
\midrule
RW-PCA vs positional PCA & $.41,\,.07$ & $.36,\,.23$ & $.16,\,.10$ & $.65,\,.03$ & $.24,\,.02$ & $.80,\,.01$ \\
\bottomrule
\end{tabular}
\end{table}

\begin{table}[h]
\centering
\caption{\textbf{Deletion dose on the rotary models.} Percentage of the clean induction gain destroyed and Pile $\Delta$NLL (nats) when the top two, three, and four bands by $\operatorname{PosRatio}$ are deleted. The two-band column is the RW-PCA column of Table~\ref{tab:cs2rep}.}
\label{tab:dose}
\footnotesize
\begin{tabular}{lccc}
\toprule
Model & Top 2 bands & Top 3 bands & Top 4 bands \\
\midrule
Pythia-160m & $95.8\%$ / $+8.1$ & $95.6\%$ / $+9.5$ & $96.2\%$ / $+10.6$ \\
Pythia-2.8B & $85.9\%$ / $+9.7$ & $97.7\%$ / $+10.2$ & $100.1\%$ / $+11.0$ \\
Pythia-6.9B & $89.5\%$ / $+8.6$ & $85.2\%$ / $+11.4$ & $86.1\%$ / $+12.4$ \\
\bottomrule
\end{tabular}
\end{table}
\begin{table}[h]
\centering
\caption{\textbf{First-block / second-block losses (nats) behind Table~\ref{tab:cs2rep}.} Mean next-token loss on the first (random) block and on the second (repeated) block of the induction corpus, clean and after each two-dimensional deletion. Random is the median over the five random deletions. The deletion experiments draw their own random blocks, so the clean losses differ from those of Table~\ref{tab:app1losses}.}
\label{tab:blocklosses}
\footnotesize
\setlength{\tabcolsep}{3pt}
\begin{tabular}{lcccccc}
\toprule
Model & Clean & RW-PCA (ours) & Activation PCA & Outlier dims & Positional PCA & Random \\
\midrule
GPT-2 & $12.56$ / $0.25$ & $12.84$ / $11.89$ & $18.40$ / $16.84$ & $18.44$ / $13.20$ & $12.60$ / $4.14$ & $12.58$ / $0.27$ \\
GPT-2-medium & $12.57$ / $0.23$ & $13.69$ / $11.54$ & $19.88$ / $18.46$ & $16.32$ / $12.77$ & $12.54$ / $0.49$ & $12.56$ / $0.25$ \\
GPT-2-large & $12.54$ / $0.26$ & $12.48$ / $10.25$ & $15.68$ / $14.03$ & $15.08$ / $0.46$ & $12.54$ / $0.50$ & $12.56$ / $0.26$ \\
Pythia-160m & $28.15$ / $9.79$ & $30.95$ / $30.18$ & $29.34$ / $18.68$ & $29.69$ / $17.83$ & $29.18$ / $16.70$ & $28.21$ / $10.18$ \\
Pythia-2.8B & $12.76$ / $0.47$ & $15.63$ / $13.90$ & $14.12$ / $11.41$ & $14.57$ / $12.55$ & $14.66$ / $13.18$ & $12.76$ / $0.48$ \\
Pythia-6.9B & $12.76$ / $0.34$ & $16.78$ / $15.47$ & $13.79$ / $11.87$ & $14.24$ / $6.36$ & $13.57$ / $9.61$ & $12.75$ / $0.35$ \\
\bottomrule
\end{tabular}
\end{table}

\section{Glossary}
\label{app:glossary}
{\footnotesize
\begin{longtable}{lp{9.6cm}}
\toprule
Term & Definition \\
\midrule
\endfirsthead
\toprule
Term & Definition \\
\midrule
\endhead
residual stream & the $d$-dimensional per-position vector every component reads from and adds to (Eq.~\ref{eq:additivity}) \\
QK / OV circuit & a head's pair-scoring form $W_{QK,h}$ / content-transforming operator $W_{OV,h}$ \\
K/Q/V-composition & the three channels by which one head's writes enter another's keys / queries / values \\
potential connectivity & write and read subspaces aligned in the weights, regardless of whether any input uses the channel \\
edge; wire & a head pair selected against the empirical null (FDR $q = 0.05$); a neuron$\to$neuron pair with $|\cos| \ge 0.5$ \\
bands & leading eigendirections of the pooled coupling matrix, the directions most components read and write \\
theoretical null distribution & a pair's coupling distribution under Haar rotation of one side, singular values fixed (Prop.~\ref{prop:chance}, Eq.~\ref{eq:isoensemble}) \\
empirical null distribution & the stratum's null distribution, fitted robustly from the observed scores (Eq.~\ref{eq:empnull}); the standard by which edges are selected \\
induction gain & first-block minus second-block loss on repeated random blocks, in nats (Eq.~\ref{eq:gain}) \\
RW-PCA & reader--writer PCA, the eigendecomposition of the pooled coupling matrix $S$ (Section~\ref{app:bandselect}), PCA applied to the model's read/write directions rather than to activations \\
positional coupling & $C^{2}_{\mathrm{pos}}(v) = \lVert \Pi^{\top} v \rVert^2 / \lVert \Pi \rVert_F^2$, a direction's squared coupling with the positional matrix $\Pi$ (Section~\ref{app:bandselect}). The token coupling $C^{2}_{\mathrm{tok}}$ is the analog with $W_E$ \\
$\operatorname{PosRatio}$ & a band's squared positional coupling divided by its squared token coupling (Section~\ref{app:bandselect}). The application deletes the top two bands by this ratio \\
\bottomrule
\end{longtable}}

\end{document}